%% file: main.tex
\documentclass{article}
\PassOptionsToPackage{numbers}{natbib}
\usepackage[preprint]{neurips_2026} 

\usepackage[utf8]{inputenc}
\usepackage[T1]{fontenc}
\usepackage{hyperref}
\usepackage{url}
\usepackage{booktabs}
\usepackage{amsfonts}
\usepackage{amssymb}
\usepackage{amsmath}
\usepackage{algorithm}
\usepackage{algpseudocode}
\usepackage{nicefrac}
\usepackage{microtype}
\usepackage{enumitem}
\usepackage{float}
\usepackage{multirow}
\usepackage{pgfplots}
\pgfplotsset{compat=1.18}
\usepackage{amsthm}
\newtheorem{theorem}{Theorem}

\usepackage[dvipsnames]{xcolor}
\usepackage{cleveref}
\usepackage{makecell}
\usepackage{subcaption}
\usepackage{graphicx}
\usepackage{longtable}
\usepackage{wrapfig}

\usepackage{dsfont}

\newcommand{\stam}[1]{}

\usepackage{mathtools}

\newcommand{\bx}{\mathbf{x}}

\newcommand{\bu}{\mathbf{u}}
\newcommand{\bv}{\mathbf{v}}

\newcommand{\bz}{\mathbf{z}}

\newcommand{\br}{\mathbf{r}}
\newcommand{\bm}{\mathbf{m}}

\newcommand{\bp}{\mathbf{p}}
\newcommand{\bh}{\mathbf{h}}
\newcommand{\by}{\mathbf{y}}

\newcommand{\bq}{\mathbf{q}}

\newcommand{\btheta}{{\boldsymbol{\theta}}}

\newcommand{\reals}{{\mathbb R}}

\newcommand{\inner}[1]{\langle #1 \rangle}

\definecolor{Green4}{RGB}{0,139,0}                
\definecolor{myblue}{HTML}{1F77B4}
\definecolor{mygreen}{HTML}{2CA02C}
\definecolor{myred}{HTML}{D62728}
\definecolor{mypurple}{HTML}{9467BD}
\newcommand{\sgreen}[1]{\textcolor{Green4}{+#1\%}}

\newcommand{\RR}{\mathbb{R}}

\title{STARFISH: faST Accuracy Recovery in pruned networks From Internal State Healing}
\author{%
  Shir Maon \thanks{Weizmann Institute of Science, \texttt{shir.maon@weizmann.ac.il}} \\
  \And
  Odelia Melamed \thanks{Weizmann Institute of Science, \texttt{odelia.melamed@weizmann.ac.il}}\\
  \And
  Adi Shamir \thanks{Weizmann Institute of Science, \texttt{adi.shamir@weizmann.ac.il}}\\
}

\begin{document}

\maketitle

\begin{abstract}
Pruning is a process designed to reduce the number of weights in a large neural network. This can substantially speed up inference but might cause a considerable reduction in the model's accuracy, and thus it is usually followed by a healing process that regains some of the lost accuracy. In this paper, we propose a new healing method, \emph{STARFISH}, that can recover (most of) the accuracy of any pruned network efficiently. The main idea of STARFISH is to optimize the pruned network to align with the original network's internal state representations using a tiny calibration set of unlabeled examples. For the common case of removing $50\%$ of the weights, STARFISH healing improves the recovered accuracy by up to $22\%$ over the state-of-the-art methods on ViT-based networks. Its advantage is even more pronounced under aggressive pruning. For example, after eliminating $75\%$ of the weights in a DeiT-B network for ImageNet, STARFISH uses only $0.4\%$ of the number of training images as a calibration set and recovers $82\%$ of the original dense accuracy, whereas competing recovery techniques reach only $40\%$ of the dense model accuracy. Open-source code is available \href{https://github.com/shrikboi/STARFISH}{here}.

\end{abstract}

\section{Introduction}

Modern neural networks achieve increasingly strong performance, often at the cost of a growing parameter count. This growth makes deployment expensive in terms of memory, latency, and energy consumption. \emph{Pruning} addresses this problem by removing some weights, neurons, channels, attention heads, or other computational units from a pretrained model, producing a smaller and more efficient network while attempting to preserve the predictive performance of the original dense model. 

When performing pruning, the desired \emph{sparsity}, namely the fraction of removed weights, plays a key role. While low sparsity saves only a limited amount of space and computation, high sparsity pruning often causes a substantial drop in performance.
Early pruning and recovery breakthroughs report high accuracy with very high sparsity, removing up to $98.5\%$ of the weights in convolutional neural networks (CNNs; \cite{cnn})~\cite{frankle2019stabilizing}. Yet, modern large-scale pruning and recovery methods struggle to restore accuracy after high-sparsity pruning, and excel in lower-sparsity settings of $50\%$ of weights removed.

Though it has been shown that an optimal pruning might not need recovery \cite{malach2020proving, ramanujan2020s}, many pruning pipelines include a recovery procedure for the remaining weights, which is often the most expensive part of the pipeline. This line of work started in the simplest pruning method followed by retraining or fine-tuning of the remaining weights, often through multiple steps of low-sparsity pruning and healing schedules until the target sparsity is reached ~\cite{han2015learning,zhu2017prune}. Inspired by the Lottery Ticket Hypothesis~\cite{frankle2018lottery}, other iterative pruning variants include rewinding the surviving weights to an early training checkpoint, before retraining the sparse model~\cite{frankle2019stabilizing,renda2020comparing}.

As neural networks and the corresponding datasets grow larger, these procedures require substantial computing resources. In addition, they require access to the original training data and knowledge of its randomness, which are often unrealistic in modern settings, where training details are typically treated as proprietary knowledge, even in open-source models. Thus, a line of work of efficiency-motivated methods, focuses on avoiding full retraining and uses local curvature restoration~\cite{hassibi1993second}, which was later extended to approximate curvature restoration~\cite{meng2024falcon,benbaki2023fast,singh2020woodfisher,yu2022combinatorial,frantar2023sparsegpt}. The latest recovery-focused methods aim to repair the pruned model directly, using a small calibration set and fast local approximations of the pruning damage based on the networks' hidden states. Two such methods are SNOWS~\cite{lucas2024preserving}, which performs layerwise recovery by preserving downstream representations using Hessian approximation, and the very recent CORP~\cite{zhang2026corp} (published in February 2026), which derives a closed-form representation recovery function for pruning in vision transformers.

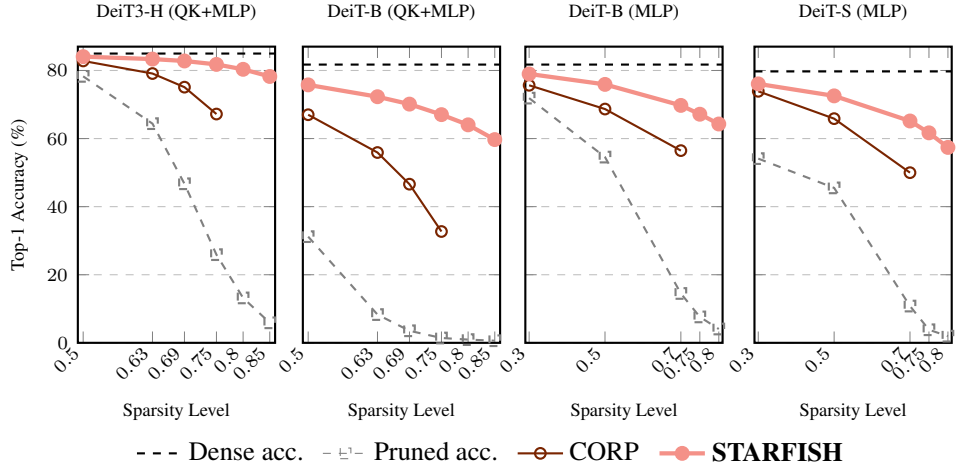
\begin{figure*}[t]\label{fig:highsparse}
\centering
\begin{subfigure}[t]{0.3\linewidth}
\centering
\begin{tikzpicture}
\begin{axis}[
    title={DeiT3-H (QK+MLP)},
    title style={font=\scriptsize},
    ylabel={Top-1 Accuracy (\%)},
    xlabel={Sparsity Level},
    xmin=0.49, xmax=0.86,
    ymax=87, ymin=0, yticklabel style={font=\scriptsize}, 
    xticklabel style={rotate=45, anchor=east, font=\scriptsize},
    xtick={0.5,0.63,0.69,0.75,0.8,0.85},
    legend to name=commonlegend,
    legend columns=5,
    legend style={draw=none, /tikz/every even column/.append style={column sep=0.5em}},
    ymajorgrids=true,
    grid style=dashed,
    width=\linewidth,
    height=5.5cm,
    thick,
    font=\scriptsize
]
\addplot[color=black, dashed, thick] coordinates {
    (0.49,84.97) (0.86,84.97)
};
\addlegendentry{Dense acc.}

\addplot[color=gray, mark=square, dashed] coordinates {
    (0.5,78.27) (0.63,64.38) (0.69,46.72) (0.75,25.90) (0.8,13.19) (0.85,5.88)
};
\addlegendentry{Pruned acc.}

\addplot[color=Brown, mark=o, thick] coordinates {
    (0.5,82.80) (0.63,79.10) (0.69,75.10) (0.75,67.20)
};
\addlegendentry{CORP}

\addplot[color=Salmon, mark=*, ultra thick] coordinates {
    (0.5,84.06) (0.63,83.35) (0.69,82.79) (0.75,81.78) (0.8,80.34) (0.85,78.24)
};
\addlegendentry{\textbf{STARFISH}}

\end{axis}
\end{tikzpicture}
\label{fig:lizard_deit_h_qkv_mlp}
\end{subfigure}%
\hspace{-3em}
\begin{subfigure}[t]{0.3\linewidth}
\centering
\begin{tikzpicture}
\begin{axis}[
    title={DeiT-B (QK+MLP)},
    title style={font=\scriptsize},
    xlabel={Sparsity Level},
    xmin=0.49, xmax=0.86,
    ymax=87, ymin=0,
    xticklabel style={rotate=45, anchor=east, font=\scriptsize},
    xtick={0.5,0.63,0.69,0.75,0.8,0.85},
    ymajorgrids=true,
    grid style=dashed,
    width=\linewidth,
    height=5.5cm,
    thick,
    yticklabels=\empty,
    font=\scriptsize
]
\addplot[color=black, dashed, thick] coordinates {
    (0.49,81.73) (0.86,81.73)
};

\addplot[color=gray, mark=square, dashed] coordinates {
    (0.5,31.19) (0.63,8.30) (0.69,3.54) (0.75,1.42) (0.8,0.95) (0.85,0.62)
};

\addplot[color=Brown, mark=o, thick] coordinates {
    (0.5,67.00) (0.63,55.90) (0.69,46.60) (0.75,32.70)
};

\addplot[color=Salmon, mark=*, ultra thick] coordinates {
    (0.5,75.77) (0.63,72.29) (0.69,70.12) (0.75,67.03) (0.8,64.03) (0.85,59.70)
};
\end{axis}

\end{tikzpicture}
\label{fig:lizard_deit_b_qkv_mlp}
\end{subfigure}%
\hspace{-3.8em}
\begin{subfigure}[t]{0.3\linewidth}
\centering
\begin{tikzpicture}
\begin{axis}[
    title={DeiT-B (MLP)},
    title style={font=\scriptsize},
    xlabel={Sparsity Level},
    xmin=0.29, xmax=0.81,
    ymax=87, ymin=0,
    yticklabel style={font=\scriptsize},
    yticklabels=\empty,
    xticklabel style={rotate=45, anchor=east, font=\scriptsize},
    xtick={0.3,0.5,0.7,0.75,0.8},
    ymajorgrids=true,
    grid style=dashed,
    width=\linewidth,
    height=5.5cm,
    thick,
    font=\scriptsize
]
\addplot[color=black, dashed, thick] coordinates {
    (0.29,81.73) (0.81,81.73)
};

\addplot[color=gray, mark=square, dashed] coordinates {
    (0.3,71.88) (0.5,54.56) (0.7,14.50) (0.75,7.61) (0.8,4.06)
};

\addplot[color=Brown, mark=o, thick] coordinates {
    (0.3,75.65) (0.5,68.67) (0.7,56.46)
};

\addplot[color=Salmon, mark=*, ultra thick] coordinates {
    (0.3,78.98) (0.5,75.91) (0.7,69.74) (0.75,67.15) (0.8,64.28)
};
\end{axis}
\end{tikzpicture}
\label{fig:lizard_deit_b_mlp}
\end{subfigure}
\hspace{-3.8em}
\begin{subfigure}[t]{0.3\linewidth}
\centering
\begin{tikzpicture}
\begin{axis}[
    title={DeiT-S (MLP)},
    title style={font=\scriptsize},
    xlabel={Sparsity Level},
    xmin=0.29, xmax=0.81,
    ymax=87, ymin=0, yticklabel style={font=\scriptsize},   yticklabels=\empty,
    xticklabel style={rotate=45, anchor=east, font=\scriptsize},
    xtick={0.3,0.5,0.7,0.75,0.8},
    ymajorgrids=true,
    grid style=dashed,
    width=\linewidth,
    height=5.5cm,
    thick,
    font=\scriptsize
]
\addplot[color=black, dashed, thick] coordinates {
    (0.29,79.72) (0.81,79.72)
};

\addplot[color=gray, mark=square, dashed] coordinates {
    (0.3,54.11) (0.5,45.51) (0.7,10.80) (0.75,3.84) (0.8,1.95)
};

\addplot[color=Brown, mark=o, thick] coordinates {
    (0.3,73.88) (0.5,65.85) (0.7,49.98)
};

\addplot[color=Salmon, mark=*, ultra thick] coordinates {
    (0.3,76.06) (0.5,72.54) (0.7,65.16) (0.75,61.68) (0.8,57.40)
};
\end{axis}
\end{tikzpicture}
\label{fig:lizard_deit_s_mlp}
\end{subfigure}%

\centering
\ref*{commonlegend}

\caption{Comparison of the top-1 accuracy results after applying the STARFISH and CORP healing processes on different pruning scopes and sparsity levels for the DeiT-B, DeiT3-H, and DeiT-S architectures trained on ImageNet, both following CORP structured pruning. STARFISH recovery is more successful at all sparsity levels, reaching an even larger advantage as sparsity increases.}
\label{fig:lizard_high_sparsity_corp}
\end{figure*}

In this paper, we present the \textbf{STARFISH} (fa\textbf{ST} \textbf{A}ccuracy \textbf{R}ecovery in pruned
networks \textbf{F}rom \textbf{I}nternal \textbf{S}tate \textbf{H}ealing) method. The starfish is a marine animal with exceptional healing abilities that can regrow multiple arms even when completely cut off. Similarly, our STARFISH method excels at healing the performance even after the network undergoes aggressive pruning, regardless of the pruning method or type. We show that optimizing the directional alignment of the intermediate representations of the pruned network with those of the original network is a good proxy for accuracy restoration. 
Our contribution is built on similar objectives that were previously shown to better capture the models' knowledge in the fields of knowledge distillation, self-supervised and unsupervised learning ~\cite{romero2015fitnetshintsdeepnets,sanh2019distilbert,zagoruyko2017paying,tian2020contrastive,chen2020simple,grill2020bootstrap,ericsson2022self}, now using it as a simple post-pruning recovery mechanism, empirically showing that this objective is especially effective for high-sparsity pruning.

In \Cref{fig:lizard_high_sparsity_corp}, we demonstrate this remarkable improvement for four modern transformer-based image classifiers ~\cite{touvron2021training,touvron2022deit} at various sparsity levels: our STARFISH method (in pink) is much better than the state-of-the-art recovery method CORP (in brown), especially in the most interesting regime of high sparsity pruning, where we denote the sparsity level as the fraction of weights removed. For the sake of completeness, we include the baseline accuracy when no healing is used (in dashed grey).

In more detail, our main contributions are:
\begin{enumerate}
    \item We present STARFISH, an efficient recovery method based on intermediate-representation alignment, which is independent of pruning types, methods, or network architecture.
    Building on previous hidden-state alignment research, we empirically show that optimizing it over the calibration set is increasingly more effective at improving test accuracy than direct output-based optimization as the sparsity level increases.
    \item We show experimentally, on large image classifiers such as ViT~\cite{dosovitskiy2020image} and MobileNet ~\cite{howard2017mobilenets}, that
our method is the first to efficiently restore even $92\%$ of the original accuracy at the very high sparsity level of $0.85$, in which the pruning process almost completely degrades the original accuracy. 
\item For the more modest sparsity of $0.5$ across different ViT-based architectures ~\cite{touvron2021training,touvron2022deit}, we show that our method outperforms the competing methods by a large factor of up to $31.5$ percentage-points improvement in recovering the original network accuracy, reaching up to $99.8\%$ accuracy retention of the dense model accuracy. 
\end{enumerate}

\section{Related work}
Trying to create more efficient machine learning models, the pruning line of work involves the removal of weights from the network while attempting to preserve the original accuracy. While theoretical works demonstrate the existence of optimal pruning that requires no recovery at all \cite{ramanujan2020s, malach2020proving}, much of the pruning literature focuses on optimizing the pruning method and using fine-tuning, or equivalent training-length methods for recovery~\cite{lecun1989optimal, han2015learning, zhu2017prune, zheng2022savit, 
kuznedelev2023cap,
chen2025optimal}. Later works observed that rewinding the remaining weights to an early checkpoint and retraining works better~\cite{frankle2018lottery,renda2020comparing}. Since such recovery procedures can be computationally expensive, another line of work seeks to restore performance without full retraining or rewinding. A classical example is Optimal Brain Surgeon (OBS)~\cite{hassibi1993second}, which updates the remaining weights to recover the lost curvature caused by the pruning. 

Modern methods scale these principles to large models using efficient curvature estimates. WoodFisher~\cite{singh2020woodfisher} accelerates OBS by approximating the inverse Hessian with the empirical Fisher matrix, while CBS~\cite{yu2022combinatorial} improves efficiency further through simultaneous weight removal rather than iterative pruning. These ideas were extended to large language models: SparseGPT~\cite{frantar2023sparsegpt} adapts the OBS pruning and update rule to a local layerwise problem, and in~\cite{kwon2022fast}, the authors propose transformer pruning via Fisher-based mask search, rearrangement, and layerwise weight updates. CHITA~\cite{benbaki2023fast} further exploits the low rank of the Fisher matrix by reformulating pruning as an $\ell_0$ sparse linear regression problem, while FALCON~\cite{meng2024falcon} extends this line of work with a FLOP-cost constraint.

As model sizes increase, even approximating the entire model's geometric properties becomes computationally expensive, motivating a shift to explicitly examining layerwise computations and outputs, commonly referred to as hidden representations, using a small calibration set.
SNOWS~\cite{lucas2024preserving} performs layerwise representation reconstruction using a Hessian approximation to perform a Newton descent optimization step, while CORP~\cite{zhang2026corp} uses a one-shot approximation of an affine transformation between the original and pruned representations in vision transformers.

\paragraph{Representation alignment.}
Representation alignment has been widely studied in knowledge distillation~\cite{hinton2015distilling}, where intermediate teacher representations are used to provide stronger supervision than output matching alone, including hidden-state hints, attention maps, language-model hidden-state alignment, and contrastive feature matching~\cite{romero2015fitnetshintsdeepnets,zagoruyko2017paying,sanh2019distilbert,tian2020contrastive}. 
A related line of work uses representation alignment in self-supervised and unsupervised learning, where models learn useful embeddings by aligning related views, target networks, or even neural network weights without relying on manual labels~\cite{chen2020simple,grill2020bootstrap,ericsson2022self}. 
More broadly, representation similarity measures, such as cosine-based comparisons, have been shown to better capture cross-network similarities~\cite{kornblith2019similarity}.

\section{Method}
\label{sec:method}

\subsection{Problem formulation}
\label{sec:method:problem}

Given a neural network, referred to as the dense model, pruning is the process of sparsifying its parameters in order to reduce memory, computation, or inference cost. 
Let $\mathcal{D}$ denote the data distribution from which we sample the labeled data
$(\bx,y)\in\RR^d\times[K]$, where $K$ is the number of classes. 
Let $f_{\btheta}: \RR^d\to\RR^{K}$ be a pretrained $K$-class classifier with parameters $\btheta$,
for which we are given a binary pruning mask $\bm \in \{0,1\}^{|\btheta|}$ resulting in a pruned model $f_{\btheta \odot \bm}$, where $\odot$ denotes elementwise multiplication. For any classifier
$g:\RR^d\to\RR^{K}$, predictions are given by
$\arg\max_{k\in[K]} g(\bx)_k$, and we define
\[
\mathrm{Acc}_{\mathcal D}(g)
=
\Pr_{(\bx,y)\sim\mathcal D}
\left[
\arg\max_{k\in[K]} g(\bx)_k = y
\right].
\]

As pruning typically degrades performance, the recovery goal is to restore this lost performance without changing the mask. In particular, we seek to obtain parameters $\widetilde\btheta$ such that the recovered pruned model $f_{\widetilde\btheta\odot\bm}$ matches the predictive accuracy of the dense model
\[    \mathrm{Acc}_{\mathcal{D}}(f_{\widetilde\btheta\odot\bm}) \approx \mathrm{Acc}_{\mathcal{D}}(f_{\btheta}).\]

\subsection{The STARFISH algorithm}
We present the STARFISH method, which aims to encourage the pruned model to align with the dense model's intermediate representations across the network, rather than matching only the final output, by a simple optimization on a small calibration set. The procedure consists of two steps. Let $S_{\mathrm{cal}} = \{\bx_1,\ldots,\bx_n\}$ be a small unlabeled calibration set sampled from $\mathcal{D}$ such that for all $i$, $\bx_i \notin S_{\mathrm{train}}$, and $|S_{\mathrm{cal}}| \ll |S_{\mathrm{train}}|$.
First, we run the dense model on $S_{\mathrm{cal}}$ and cache the target intermediate representations. Second, 
we train $f_{\btheta \odot \mathbf{m}}$ to align with those cached representations.

For a formal algorithm description, we let L be the number of blocks in $f_\btheta$. Given an input $\bx \in \reals^d$, for each $\ell\in [L]$, we define $\br_\bx^{\ell}$ as the vectorized output of the dense model $f_{\btheta}$ after block $\ell$ (i.e., taken after all operations that form the block output, such as normalization, nonlinearities, residual additions, and other architecture-specific post-processing). Denote by $f_{\btheta \odot \mathbf{m}}^\ell(\bx)$ the corresponding $\ell$-th representation of the pruned model $f_{\btheta \odot \bm}$. We define the alignment loss as

\begin{equation}
\label{eq:align_cos}
\mathcal{L}_{\mathrm{align}}(\btheta \odot \mathbf{m}, \bx)
=
\frac{1}{L}
\sum_{\ell \in [L]}
\left(
1 -
\frac{
\left\langle
f_{\btheta \odot \mathbf{m}}^\ell(\bx),
\br_\bx^{\ell}
\right\rangle
}{
\left\|f_{\btheta \odot \mathbf{m}}^\ell(\bx) \right\|
\left\|\br_\bx^{\ell}\right\|
}
\right),
\end{equation}

and train the pruned model $f_{\btheta \odot \bm}$ by minimizing the empirical loss over $S_{\mathrm{cal}}$, namely,
\begin{equation*}
L(\btheta) = \frac{1}{n} \sum_{i \in [n]} \mathcal{L}_{\mathrm{align}}\left(\btheta \odot \mathbf{m}, \bx_i\right).
\end{equation*}

We note that the STARFISH recovery method offers two significant advantages over existing methods. First, the number of samples in $S_{\mathrm{cal}}$ is extremely small compared to the train set, and does not need to be drawn from it. 
This is particularly desirable in modern ML settings, where the trained model might be open source, while the original training set is not necessarily accessible.
Second, while different recovery methods are often dependent and even coupled with a specific pruning method, STARFISH adjusts the remaining weights after the damage has occurred, independently of the pruning. This allows the STARFISH recovery method to be easily applicable to any pruning method, sparsity pattern, pruning scope, and architecture. 

As representation alignment has been found useful in other research fields, such as knowledge distillation and unsupervised learning, for aligning networks' internal computations, the STARFISH recovery method utilizes it to recover highly pruned networks.
In particular, in \Cref{sec:method:theory}, we show that the STARFISH method, leveraging representation alignment, achieves better performance than output-based objectives for recovering severely pruned networks.

\begin{figure}[ht]
  \centering \includegraphics[width=1\textwidth]{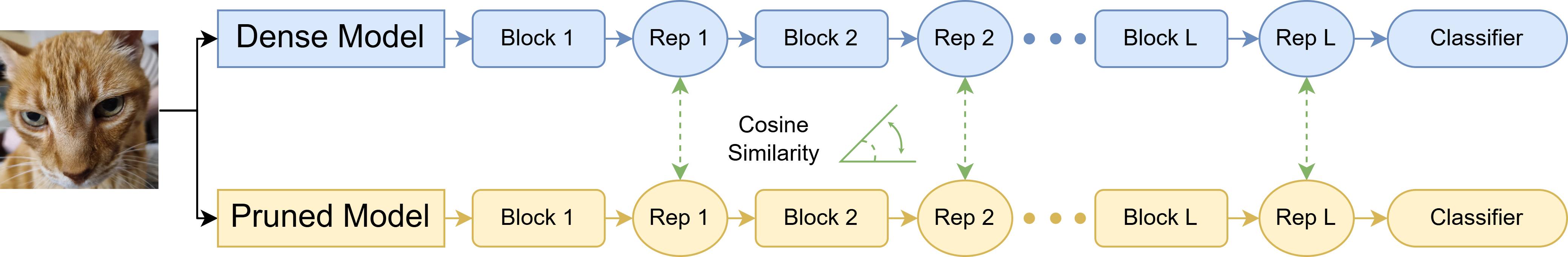} 
  \caption{Visualization of the representation computation and alignment of STARFISH, where the cosine similarity is computed using the hidden representation taken after each corresponding computation block.}
  \label{fig:algorithm}
\end{figure}

\subsection{Computational complexity}\label{sec:method:computational}
Compared to similar optimization-based recovery methods that use the entire training set, the STARFISH method is extremely efficient, requiring only a minimal number of forward and backward passes on a small calibration set.
In particular, STARFISH has two stages: a one-time, dense forward pass on $S_{\mathrm{cal}}$ to cache target representations, followed by a small number of recovery training epochs \(T\) on the pruned model.

Specifically, let \(n=|S_{\mathrm{cal}}|\) be the number of calibration samples, \(P\) be the cost of one dense forward or backward pass, and $\widetilde P$ be the equivalent for the pruned network (potentially, $\widetilde P < P$).
Then, since the loss computation is negligible, the time complexity of the STARFISH method is
\[
O\left(nP + Tn\widetilde P\right).
\]
In this paper, across all architectures and comparisons, we use very small values of $n \leq 9{,}000$ and $T \leq 10$, achieving state-of-the-art recovery results.

\paragraph{Comparison with other efficient recovery methods.}
The CORP~\cite{zhang2026corp} recovery method avoids iterative optimization by using a closed-form layerwise recovery for a specific sparsity pattern on ViTs. 
Under this restricted setting, we analyze the complexity where \(\rho\) is the sparsity level, i.e., (\(1-\rho\))-fraction of weights remain active and optimally, $\widetilde P = (1-\rho)P$. Let \(L\) be the number of blocks, \(Q\) be the ViT sequence length and \(d_t\) be a single token dimension. 
Including the dense calibration pass, the total simplified cost of CORP is
\[
O\left(nP + L(1-\rho)^3d_t^3 + nQL(1-\rho)^2d_t^2\right).
\]
We note that $\left(L(1-\rho)^3d_t^3 + nQL(1-\rho)^2d_t^2\right)$ is the overhead of computing the closed form solution, which is comparable to a single forward or backward pass. Thus, the total time complexity of STARFISH is comparable to that of CORP. 

The SNOWS method ~\cite{lucas2024preserving} is more general than CORP and, like STARFISH, can recover a fixed-mask model by matching intermediate representations. However, SNOWS performs local Hessian approximation computation using gradients for every $C$-layer sliding window, and for $I$ Newton steps, thus its ViT recovery cost is approximately
\[
O\left(ICn \widetilde P\right).
\]

In practice, STARFISH recovers a $0.8$-sparse DeiT3-H on a single A10 GPU in 16 minutes using 1,000 examples and 5 epochs, reaching a final \(73\%\) accuracy. We note that even when further reducing the calibration set size to only 500 examples---where there is less than one example per class, and STARFISH completes recovery in 8 minutes---it still outperforms competing methods evaluated at lower sparsity levels. 
Additional runtime results are reported in \Cref{app:runtime}.

\section{Representation alignment}
\label{sec:method:theory}
In this section, we discuss the advantage of intermediate representation similarity optimization over regular output-based optimization.
The STARFISH method treats pruning recovery as internal-state recovery,
rather than matching only the final outputs of the dense model and the pruned model, thus
encouraging the pruned model to follow the dense model's computation throughout the network.

First, we motivate the representation-based optimization by proving that the distance between the representations used as inputs to the classification head is an upper bound on the classification head's output KL-divergence.
We note that even when the classification head is not pruned, while the other weights are, the head's input representation in the pruned network can differ from that in the original network, necessitating recovery. This result further supports prior research on the contribution and meaning of representation alignment in the context of severe pruning recovery.

Formally, for an input $\bx_i \in S_{\text{cal}}$, we denote by $\bh_i, \widetilde \bh_i \in \RR^{d_{\mathrm{out}}}$ the last hidden representations (i.e., the input of the classification head) of the dense and recovered pruned models, respectively. 
Let $A \in \RR^{d_{\mathrm{out}} \times K}$ and $\widetilde A \in \RR^{d_{\mathrm{out}} \times K}$ be the dense and recovered pruned classification heads, where $K$ is the number of classes, and denote by $\bp_i, \widetilde \bp_i$ the corresponding predictive distribution vectors. We prove that 
\[
\frac{1}{n}\sum_{i=1}^n
\mathrm{KL}(\bp_i \,\|\, \widetilde \bp_i)
\le
\frac{1}{2n}
\sum_{i=1}^n\zeta_i
\left(
\|(\bh_i - \widetilde  \bh_i )^\top A\|_2
+
\|\widetilde \bh_i^\top (\widetilde A -  A)\|_2
\right)^2 . 
\] 

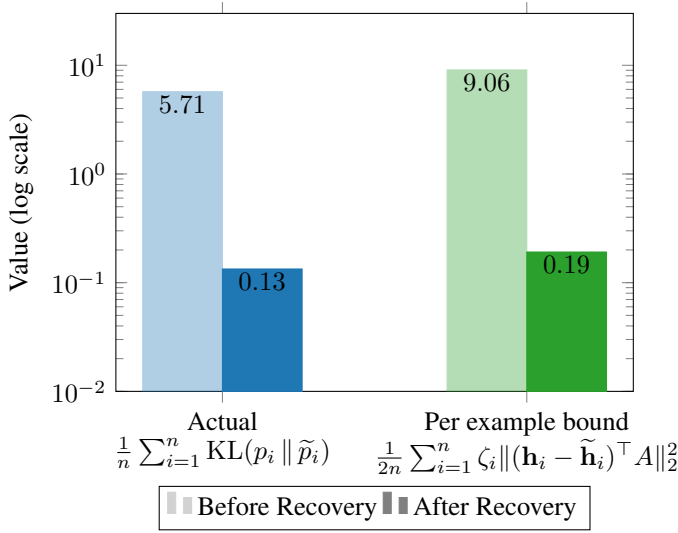
\begin{wrapfigure}[25]{l}{0.7\textwidth}
\centering
\begin{tikzpicture}[trim axis left, trim axis right]
\begin{axis}[
    ybar=0.05,
    ymode=log,
    ymin=1e-2, ymax = 30,
    log origin=infty, 
    ylabel={Value (log scale)},
    width=0.7\linewidth,
    height=5cm,
    bar width=30pt,
    xtick={1,2},
    xticklabels={
    {Actual\\$\frac{1}{n}\sum_{i=1}^n \mathrm{KL}(p_i \,\|\, \widetilde p_i)$},
    {Per example bound \\$\frac{1}{2n}
    \sum_{i=1}^n \zeta_i
    \|(\bh_i - \widetilde  \bh_i )^\top A\|_2^2$}
    },
    xticklabel style={
        align=center,
    },
    enlarge x limits=0.35,
    legend style={
        at={(0.5,-0.25)},
        anchor=north,
        legend columns=2
    },
    scale only axis,
    nodes near coords={
        \pgfmathprintnumber[fixed,precision=2]{\pgfplotspointmeta}
    },
    point meta=rawy,
    every node near coord/.append style={
        anchor=north,
        yshift=2pt, color=black, opacity=1
    },
    clip=false
]

\addlegendimage{ybar, fill=gray, draw=gray, fill opacity=0.35, draw opacity=0.35}
\addlegendentry{Before Recovery}
\addlegendimage{ybar, fill=gray, draw=gray}
\addlegendentry{After Recovery}

\addplot+[bar shift=-15pt, fill=myblue, draw=myblue, fill opacity=0.35, draw opacity=0.35]
coordinates {(1,5.7127e+00)};

\addplot+[bar shift=15pt, fill=myblue, draw=myblue]
coordinates {(1,1.3363e-01)};

\addplot+[bar shift=-15pt, fill=mygreen, draw=mygreen, fill opacity=0.35, draw opacity=0.35]
coordinates {(2,9.0598e+00)};

\addplot+[bar shift=15pt, fill=mygreen, draw=mygreen]
coordinates {(2,1.9148e-01)};

\end{axis}
\end{tikzpicture}
\caption{
    Empirical visualization of the representation-based bound for KL divergence before and after
    recovery. After STARFISH recovery, the representation error decreases, and with it both the empirical KL and the corresponding bound decrease sharply.
    }\label{fig:theory_empirical_kl}
\end{wrapfigure}

for some small parameters $\zeta_i$, dependent on the network's normalized output logits. The formal theorem and full proof are given in \Cref{app:alignment}. In \Cref{fig:theory_empirical_kl} we illustrate the main implication of
Theorem~\ref{thm:kl_representation_general} in the common case of an unpruned classification head following a high-sparsity pruned network: STARFISH recovers the final representation successfully and
substantially reduces the predictive KL divergence. For a DeiT-B model pruned with
magnitude pruning to $0.8$-sparsity and recovered with STARFISH, both the measured KL and the representation-based KL bound decrease sharply after recovery. 

We note that our choice of cosine-similarity instead of the difference norm is motivated by prior work using cosine-based criteria for representation alignment, as well as for comparing representations
\cite{sanh2019distilbert,kornblith2019similarity,chen2020simple}, and we extend the optimization for every block, as errors introduced by pruning can accumulate and amplify across depth. In \Cref{app:ablation_mse}, we empirically show that STARFISH recovery outperforms the difference-norm objective in recovery, and in \Cref{fig:theory_empirical_kl}, we demonstrate that it also decreases the difference-norm of representations.

\subsection{Recovery vs. fine-tuning}
\label{sec:experiments:loss_ablation}

To further motivate the STARFISH recovery method, we experimentally show that when only a small unlabeled calibration set is used, the STARFISH alignment objective substantially outperforms other baseline fine-tuning objectives. We demonstrate that recovering dense model representations provides a more effective approach for post-pruning recovery than the direct approach of output-level supervision.

For the output-based losses, we consider the training baseline cross-entropy (CE) loss, and the Kullback--Leibler (KL) divergence, which is a common objective for knowledge distillation.
\Cref{tab:deitt_loss_ablation} compares the methods, as well as the joint STARFISH+KL objective, 
on a DeiT-B pruned using Magnitude Pruning (MP) at different sparsity levels. It shows that cosine alignment consistently outperforms cross-entropy fine-tuning, KL divergence matching, and the joint objective, peaking at $0.8$-sparsity, where cosine alignment achieves $73.26\%$ top-1 accuracy, improving over KL divergence by $13.15$ points, over CE fine-tuning by $16.42$ points, and over STARFISH+KL by $10.26$ points. 
Further implementation details are provided in \Cref{app:experimentaldetails} and \Cref{app:recovery_vs_fine}.

\begin{table}[ht]
\centering
\caption{Comparison of top-1 accuracy after performing various recovery techniques using representation-based and output-based objectives, on a pruned DeiT-B at different sparsity levels. Values in parentheses indicate accuracy retention relative to the dense accuracy. The representation-based recovery surpasses all output-based objectives, as well as the combined alternative.}
\label{tab:deitt_loss_ablation}
\small 
\begin{tabular}{@{}l cccc @{}}
\toprule
& \multicolumn{4}{c}{\textbf{Dense Acc.: $81.73\%$}} \\
\cmidrule(l){2-5}
\textbf{Method/Sparsity} & \textbf{0.5} & \textbf{0.6} & \textbf{0.7} & \textbf{0.8} \\
\midrule
Pruned Acc. & 74.59 (91.26\%) & 56.16 (68.71\%) & 20.02 (24.49\%) & 0.83 (1.02\%)\\
\midrule
Cross-Entropy (CE) & 69.81 (85.42\%) & 70.79 (86.61\%) & 68.26 (83.52\%) & 56.84 (69.55\%) \\
Kullback–Leibler (KL) & 73.88 (90.40\%) & 73.96 (90.49\%) & 71.35 (87.30\%) & 60.11 (73.55\%) \\
STARFISH+KL & 77.46 (94.78\%) & 77.00 (94.21\%) & 74.33 (90.95\%) & 63.00 (77.08\%) \\
STARFISH & \textbf{81.34 (99.52\%)} & \textbf{80.59 (98.61\%)} & \textbf{78.91 (96.55\%)} & \textbf{73.26 (89.64\%)} \\
\bottomrule
\end{tabular}
\end{table}

\section{Experiments}
\label{sec:experiments}

In this section, we evaluate STARFISH on ImageNet-1K \cite{imagenet} classification across vision transformers and convolutional networks, comparing it to competing efficient recovery methods in high-sparsity pruning settings and in the common $0.5$-sparsity setting. To further demonstrate the independence of the STARFISH method from pruning type or method, we compare it with multiple competing methods across multiple pruning variations. 

The first variation is the sparsity pattern, including \emph{unstructured} pruning, where individual weights are removed; \emph{structured} pruning, where whole groups such as channels, neurons, heads, or blocks are removed; or \emph{semi-structured} pruning, where constraints are imposed locally, as in $N{:}M$ sparsity, which retains $N$ weights in each group of $M$ parameters. 

Another considered variation is the pruning scope, where different pruning methods target specific parts of the dense model, such as attention weights, MLP layers, or convolutional weights. Here, we use the convention of using Q, K, and V to refer to the inner weights of the attention unit corresponding to the Query, Key, and Value; Out to refer to the final linear transformation applied to the attention heads, and MLP to refer to the fully connected unit following the multi-head attention. We mention that when discussing sparsity levels, when a model is pruned to a sparsity level imposed on specific layers or components, only the prunable weights reach that final sparsity. 
 
For the calibration set, the STARFISH method uses calibration images sampled from the ImageNet unlabeled test split. Thus, it does not use any supervised signal during recovery, nor does it revisit the training set. Additionally, the calibration set is extremely small, consisting of up to $9{,}000$ images, which are only $0.7\%$ of the training set size, 
making STARFISH recovery especially lightweight.

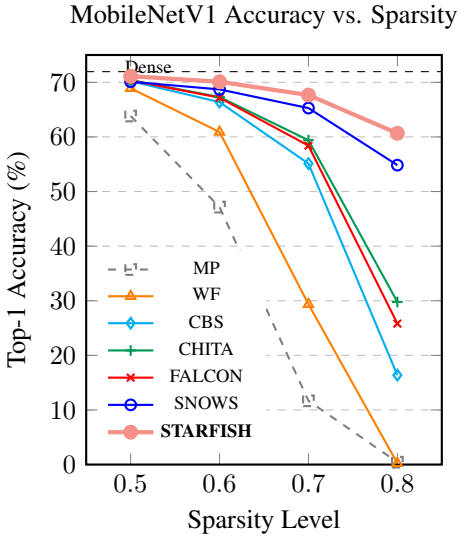
\begin{wrapfigure}[25]{r}{0.5\textwidth}
\centering
\vspace{-1.4\baselineskip}
\begin{tikzpicture}[trim axis left,
    trim axis right,
    baseline=(current axis.north)
    ]
\begin{axis}[
    title={MobileNetV1 Accuracy vs. Sparsity},
    xlabel={Sparsity Level},
    ylabel={Top-1 Accuracy (\%)},
    xmin=0.45, xmax=0.85,
    ymin=0, ymax=75,
    xtick={0.5, 0.6, 0.7, 0.8},
    ytick={0, 10, 20, 30, 40, 50, 60, 70},
    legend style={
        at={(0.03,0.03)},
        anchor=south west,
        draw=none,
        font=\scriptsize},
    ymajorgrids=true,
    grid style=dashed,
    width=0.9\linewidth,
    height=7cm,
    thick
]

\addplot[color=gray, mark=square, dashed] coordinates {
    (0.5, 63.83) (0.6, 47.15) (0.7, 11.68) (0.8, 0.43)
};
\addlegendentry{MP}

\addplot[color=orange, mark=triangle] coordinates {
    (0.5, 68.91) (0.6, 60.90) (0.7, 29.36) (0.8, 0.24)
};
\addlegendentry{WF}

\addplot[color=cyan, mark=diamond] coordinates {
    (0.5, 70.21) (0.6, 66.37) (0.7, 55.11) (0.8, 16.38)
};
\addlegendentry{CBS}

\addplot[color=ForestGreen, mark=+] coordinates {
    (0.5, 70.42) (0.6, 67.30) (0.7, 59.40) (0.8, 29.78)
};
\addlegendentry{CHITA}

\addplot[color=red, mark=x] coordinates {
    (0.5, 70.35) (0.6, 67.18) (0.7, 58.40) (0.8, 25.82)
};
\addlegendentry{FALCON}

\addplot[color=blue, mark=o, thick] coordinates {
    (0.5, 70.10) (0.6, 68.70) (0.7, 65.28) (0.8, 54.83)
};
\addlegendentry{SNOWS}

\addplot[color=Salmon, mark=*, ultra thick] coordinates {
    (0.5, 71.12) (0.6, 70.1) (0.7, 67.67) (0.8, 60.66)
};
\addlegendentry{\textbf{STARFISH}}

\addplot[dashed, black, thin] coordinates {(0.45, 71.95) (0.85, 71.95)};
\node at (axis cs:0.52,72.8) {\scriptsize Dense};

\end{axis}
\end{tikzpicture}
\caption{
Comparison of top-1 accuracy results of unstructured pruning and recovery methods on MobileNetV1. STARFISH recovery exceeds other methods, within a larger margin at high sparsity levels.}
 \label{fig:lizard_high_sparsity_mobile}
 \end{wrapfigure}

\subsection{High sparsity STARFISH}

In this section, we present STARFISH in severe pruning regimes, where recovery is most challenging and the differences between recovery methods are most pronounced. 
We compare STARFISH across sparsity levels on both MobileNetV1 and DeiT models. The advantage of STARFISH grows with sparsity: while the gains are moderate at $0.5$-sparsity, they become substantially larger once $70\text{-}85\%$ of the prunable weights are removed, where other pruning and recovery methods present smaller improvements.
For example, after recovering a $0.85$-sparse DeiT3-H model, STARFISH retains $92.08\%$ of the dense model accuracy, successfully increasing the pruned accuracy by $72.36$ points. Recovering a $0.75$-sparse DeiT-B model, STARFISH achieves accuracy with a staggering $34.33$ percentage-point improvement over competing recovery methods.

In \Cref{fig:lizard_high_sparsity_corp}, we compare STARFISH for structured pruning of DeiT-B, DeiT3-H, and DeiT-S. 
STARFISH systematically surpasses the state-of-the-art CORP recovery method~\cite{zhang2026corp}, when recovering a model pruned using the CORP mask selection method,
over the range of sparsity levels and scopes. 

A similar pattern is observed in \Cref{fig:lizard_high_sparsity_mobile}, where we compare STARFISH to several leading unstructured pruning methods on MobileNetV1. Here, STARFISH consistently outperforms the other methods across all settings.
STARFISH, like SNOWS~\cite{lucas2024preserving}, uses MP for mask selection in this comparison. In both figures, Magnitude Pruning (MP) serves as the baseline with no recovery.

\subsection{Half sparsity STARFISH}

In this section, we compare STARFISH with prior recovery methods in the common $0.5$-sparsity regime, showing that STARFISH achieves the highest recovered accuracy in every setting we evaluate.
Across different ViT models, we consider both structured pruning as well as $2{:}4$ semi-structured pruning in different pruning scopes.  For example, in the DeiT-T MLP+QK setting, pruning reduces accuracy from
$72.01\%$ to $7.57\%$, retaining only $10.51\%$ of the dense accuracy. STARFISH recovers the model to $81.95\%$ of the dense accuracy, compared with $50.41\%$ retention for CORP. This is a $31.54$ percentage-point improvement in retention. 

In \Cref{tab:structured} we compare STARFISH to other recovery methods, presenting retention percentage, denoting the recovered accuracy divided by the original accuracy, and the improvement gap between different retention scores achieved within the same setting.
In Panel A, we compare STARFISH to the CORP recovery method~\cite{zhang2026corp},
under $0.5$-sparsity structured pruning across various DeiT models, with both performed after CORP-pruning.
The comparison spans multiple pruning targets, including MLP layers, attention QK projections, and the joint MLP+QK setting. 
The gains are especially large in the more challenging joint MLP+QK setting, where pruning affects multiple components
simultaneously. 

In Panel B of \Cref{tab:structured}, we compare STARFISH with SNOWS~\cite{lucas2024preserving}, 
under $2{:}4$ semi-structured sparsity on ViT models, with both methods performed after magnitude pruning.
The comparison includes both attention-only pruning (QKV) and broader pruning scopes that additionally include output projections and MLP layers.
These results show that STARFISH consistently improves recovered accuracy beyond structured pruning and also improves recovery under hardware-friendly semi-structured sparsity. Further implementation details and extended results are provided in \Cref{app:experimentaldetails}.

\begin{table}[t!]
\centering
\setlength{\tabcolsep}{5pt}
\caption{
Top-1 accuracy improvement before and after different post-pruning recovery methods at $0.5$-sparsity. We report the \textbf{Retention} percentage, which denotes the percentage of the dense model's accuracy recovered, and the \textbf{Improvement} difference percentage, which denotes the gap between the competing retention percentages, showing that STARFISH achieves large improvement gaps over competing methods across settings. \textbf{Panel A:} Comparing STARFISH to CORP-recovery across MLP, QK, and MLP+QK structured pruning scopes on DeiT models. \textbf{Panel B:} Comparing STARFISH to SNOWS across different semi-structured pruning scopes on ViT models. Best result within each model and scope is bolded. 
}
\small
\label{tab:structured}
{%
\begin{tabular}{@{}lclcccc@{}}
\toprule
\makecell[l]{Model} 
& \makecell{Pruning\\Scope}  
& \makecell{Method} 
& \makecell{Calib.\\Size} 
& \makecell{Recovered\\Acc.}  
& \makecell{Retention\\(\%)}  
& \makecell{Improvement\\($\Delta$\%)} \\

\midrule
\multicolumn{7}{@{}l}{\textbf{Panel A: DeiT / CORP setting ($0.5$ structured sparsity)}}\\
\midrule

\multirow{6}{*}{\shortstack[l]{DeiT-T \\(Dense Acc. 72.01)}}
& \multirow{2}{*}{\shortstack{MLP\\(Pruned Acc. 10.93)}}
  &  CORP & -- & 53.20 & 73.88\% & \multirow{2}{*}{\sgreen{\textbf{10.51}}} \\
& &  STARFISH & 1,000 &  \textbf{60.77} & \textbf{84.39\%} \\
\cmidrule(l){2-7}
& \multirow{2}{*}{\shortstack{QK\\(Pruned Acc. 62.31)}}
  &  CORP & --  & 64.20 & 89.15\% & \multirow{2}{*}{\sgreen{\textbf{9.43}}} \\
& &  STARFISH & 1,000 & \textbf{70.99} & \textbf{98.58\%} \\
\cmidrule(l){2-7}
& \multirow{2}{*}{\shortstack{MLP+QK\\(Pruned Acc. 7.57)}}
  &  CORP & --  & 36.30 & 50.41\% & \multirow{2}{*}{\sgreen{\textbf{31.54}}} \\
& &  STARFISH & 1,000 &\textbf{59.01} & \textbf{81.95\%} \\
\midrule

\multirow{6}{*}{\shortstack[l]{DeiT-S \\(Dense Acc. 79.72)}}
& \multirow{2}{*}{\shortstack{MLP\\(Pruned Acc. 45.51)}}
  &  CORP & --  & 65.85 & 82.60\% & \multirow{2}{*}{\sgreen{\textbf{8.39}}} \\
& &  STARFISH & 2,500 & \textbf{72.54} & \textbf{90.99\%} \\
\cmidrule(l){2-7}
& \multirow{2}{*}{\shortstack{QK\\(Pruned Acc. 54.14)}}
  &  CORP & --  & 72.50 & 90.94\% & \multirow{2}{*}{\sgreen{\textbf{8.33}}} \\
& &  STARFISH & 2,500  & \textbf{79.14} & \textbf{99.27\%} \\
\cmidrule(l){2-7}
& \multirow{2}{*}{\shortstack{MLP+QK\\(Pruned Acc. 17.32)}}
  &  CORP & --   & 55.50 & 69.62\% & \multirow{2}{*}{\sgreen{\textbf{20.18}}} \\
& &  STARFISH & 2,500  & \textbf{71.59} & \textbf{89.80\%} \\
\midrule

\multirow{6}{*}{\shortstack[l]{DeiT-B \\ (Dense Acc. 81.73)}}
& \multirow{2}{*}{\shortstack{MLP\\(Pruned Acc. 54.56)}}
  &  CORP & --  & 68.67 & 84.02\% & \multirow{2}{*}{\sgreen{\textbf{8.86}}} \\
& &  STARFISH & 5,000 & \textbf{75.91} & \textbf{92.88\%} \\
\cmidrule(l){2-7}
& \multirow{2}{*}{\shortstack{QK\\(Pruned Acc. 73.40)}}
  &  CORP & --  & 80.80 & 98.86\% & \multirow{2}{*}{\sgreen{\textbf{0.87}}} \\
& &  STARFISH & 5,000 & \textbf{81.51} & \textbf{99.73\%} \\
\cmidrule(l){2-7}
& \multirow{2}{*}{\shortstack{MLP+QK\\(Pruned Acc. 31.19)}}
  &  CORP & --  & 67.00 & 81.98\% & \multirow{2}{*}{\sgreen{\textbf{10.73}}} \\
& &  STARFISH & 5,000 & \textbf{75.77} & \textbf{92.71\%} \\
\midrule

\multirow{6}{*}{\shortstack[l]{DeiT3-L \\  (Dense Acc. 84.58)}}
& \multirow{2}{*}{\shortstack{MLP\\(Pruned Acc. 77.81)}}
  &  CORP & --  & 80.90 & 95.65\% & \multirow{2}{*}{\sgreen{\textbf{2.74}}} \\
& &  STARFISH & 7,000  & \textbf{83.22} & \textbf{98.39\%} \\
\cmidrule(l){2-7}
& \multirow{2}{*}{\shortstack{QK\\(Pruned Acc. 82.97)}}
  &  CORP & -- & 83.60 & 98.84\% & \multirow{2}{*}{\sgreen{\textbf{0.99}}} \\
& &  STARFISH & 7,000 & \textbf{84.44} & \textbf{99.83\%} \\
\cmidrule(l){2-7}
& \multirow{2}{*}{\shortstack{MLP+QK\\(Pruned Acc. 72.09)}}
  &  CORP & --  & 79.40 & 93.88\% & \multirow{2}{*}{\sgreen{\textbf{3.89}}} \\
& &  STARFISH & 7,000 & \textbf{82.69} & \textbf{97.77\%} \\
\midrule

\multirow{6}{*}{\shortstack[l]{DeiT3-H \\ \scriptsize (Dense Acc. 84.97)}}
& \multirow{2}{*}{\shortstack{MLP\\(Pruned Acc. 81.79)}}
  &  CORP & --   & 83.70 & 98.51\% & \multirow{2}{*}{\sgreen{\textbf{0.82}}} \\
& &  STARFISH & 9,000 & \textbf{84.40} & \textbf{99.33\%} \\
\cmidrule(l){2-7}
& \multirow{2}{*}{\shortstack{QK\\(Pruned Acc. 83.23)}}
  &  CORP & --  & 84.20 & 99.09\% & \multirow{2}{*}{\sgreen{\textbf{0.75}}}\\
& &  STARFISH & 9,000  & \textbf{84.83} & \textbf{99.84\%} \\
\cmidrule(l){2-7}
& \multirow{2}{*}{\shortstack{MLP+QK\\(Pruned Acc. 78.27)}}
  &  CORP & -- & 82.80 & 97.45\% & \multirow{2}{*}{\sgreen{\textbf{1.48}}} \\
& &  STARFISH & 9,000 & \textbf{84.06} & \textbf{98.93\%} \\
\midrule

\multicolumn{7}{@{}l}{\textbf{Panel B: ViT / SNOWS setting ($2{:}4$ semi-structured sparsity)}} \\
\midrule

\multirow{4}{*}{\shortstack[l]{ViT/B-16 \\ (Dense Acc. 80.40)}}
& \multirow{2}{*}{\shortstack{QKV\\(Pruned Acc. 78.92)}}
  & SNOWS  & 5,000
 & 79.45 & 98.82\% & \multirow{2}{*}{\sgreen{\textbf{0.94}}}   \\
& & STARFISH   & 5,000 & \textbf{80.21} & \textbf{99.76\%} \\
\cmidrule(l){2-7}
& \multirow{2}{*}{\shortstack{QKV+Out+MLP\\(Pruned Acc. 70.86)}}
  & SNOWS  & 5,000 & 76.57 & 95.24\% & \multirow{2}{*}{\sgreen{\textbf{3.11}}}\\
& & STARFISH   & 5,000 & \textbf{79.07} & \textbf{98.35\%} \\

\midrule
\multirow{4}{*}{\shortstack[l]{ViT/L-16 \\ (Dense Acc. 84.20)}}
& \multirow{2}{*}{\shortstack{QKV\\(Pruned Acc. 75.40)}}
  & SNOWS  & 20,000 & 81.01 & 96.21\% & \multirow{2}{*}{\sgreen{\textbf{2.32}}} \\
& & STARFISH   & 7,000 & \textbf{82.96} & \textbf{98.53\%} \\
\cmidrule(l){2-7}
& \multirow{2}{*}{\shortstack{MLP\\(Pruned Acc. 0.20)}}
  & SNOWS  & 20,000 & 69.71 & 82.79\% & \multirow{2}{*}{\sgreen{\textbf{3.27}}}\\
& & STARFISH   & 7,000 & \textbf{72.46} & \textbf{86.06\%} \\
\bottomrule
\end{tabular}
}
\end{table}

\section{Conclusion and future directions}

We presented STARFISH, a post-pruning recovery method that treats sparse-model repair as internal-state healing. 
Given a fixed pruning mask, STARFISH updates only the remaining weights so that the sparse model's intermediate representations match those of the dense model across all blocks. While representation alignment is known to improve other training-related tasks, STARFISH’s use of this objective yields significant performance gains toward the pruning recovery goal.
This simple objective is modular across pruning methods, sparsity patterns, pruning scopes, sparsity levels, and architectures. Because STARFISH uses a tiny calibration set that consists of images outside the training set, it is efficient and can be applied when the original training data is unavailable, requiring only the dense model weights, the pruning mask, and a small unlabeled calibration set.

Empirically, STARFISH improves recovery across structured, semi-structured, and unstructured sparsity patterns, and remains effective across different pruning methods and pruning scopes. In moderate pruning regimes, STARFISH substantially improves dense-accuracy retention; for example, after pruning DeiT-T with MLP+QK structured sparsity at $0.5$-sparsity, STARFISH recovers the model to $81.95\%$ of the dense accuracy, while competitors reach $31.54$ percentage points lower retention. In severe pruning regimes, STARFISH remains effective: for example, after pruning DeiT3-H to $0.85$-sparsity, STARFISH recovers the model to $78.24\%$ top-1 accuracy, improving by $72.36$ points over the pruned model.

Future work could extend this approach in several directions. 
First, adaptive block weighting may improve recovery by emphasizing the layers most damaged by pruning, and using this representation mismatch as a criterion for pruning itself might yield even further improvement. 
Second, applying the same alignment principle to larger-scale multimodal and language models may provide a practical route to efficient post-pruning recovery beyond image classification. Finally, this idea could be incorporated directly into training by encouraging a model and its induced sparse counterpart to maintain aligned internal representations. Such a sparsity-aware representation objective may produce dense models whose pruned versions are easier to recover, or even remain performant immediately after pruning.

\clearpage

\bibliographystyle{plainnat}
\bibliography{references}

\appendix
\input{appendix}

\newpage

\end{document}

%% file: appendix.tex
\section{Representation alignment}
\subsection{Representation bound for KL divergence}\label[appendix]{app:alignment}
We state the definitions and the theorem from \Cref{sec:method:theory}. 
For an input $\bx_i \in S_{\text{cal}}$, we denote by $\bh_i, \widetilde \bh_i \in \RR^{d_{\mathrm{out}}}$ the last hidden representations (i.e., the input of the classification head) of the dense and recovered pruned models, respectively. 
Let $A \in \RR^{d_{\mathrm{out}} \times K}$ denote the dense classification head, where $K$ is the number of classes, and let $\widetilde A \in \RR^{d_{\mathrm{out}} \times K}$ denote the
recovered pruned head. We denote the output logits vectors and their difference by 
\[
\bz_i = \bh_i ^\top A 
\qquad
\widetilde \bz_i = \widetilde\bh_i ^\top \widetilde A 
\qquad \Delta \bz_i=\widetilde  \bz_i- \bz_i,
\]
and the corresponding predictive distribution vectors by
\[
\bp_i = \operatorname{softmax}(\bz_i),
\qquad
\widetilde \bp_i = \operatorname{softmax}(\widetilde \bz_i).
\]

\begin{theorem}[Representation and head recovery controls predictive KL]
\label{thm:kl_representation_general}
For a calibration set $S_{\text{cal}} = \{\bx_i\}_{i=1}^n$, we denote \footnote{If a denominator in the definitions in \eqref{eq:defs} is zero, then so is the numerator and we define the corresponding ratio to be zero.}
\begin{align}\label{eq:defs}
M_i=
\frac{\|(\widetilde \bh_i-\bh_i)^\top A\|_2}
{\|(\widetilde \bh_i-\bh_i)^\top\|_2},
\qquad
\widetilde M_i=
\frac{\|\widetilde \bh_i^\top (\widetilde A-A)\|_2}
{\|\widetilde \bh_i^\top\|_2},
\qquad 
\zeta_i
=
\sup_{t\in[0,1]}
\frac{\Delta \bz_i^\top
\nabla^2\psi(\bz_i+t\Delta \bz_i)
\Delta \bz_i}{\|\Delta \bz_i\|_2^2}
\end{align}  
where $\psi(u)=\log \sum_{k=1}^{K} e^{u_k}$ is the log-sum-exp function. Then,
\[
\frac{1}{n}\sum_{i=1}^n
\mathrm{KL}(\bp_i \,\|\, \widetilde \bp_i)
\le
\frac{1}{2n}
\sum_{i=1}^n\zeta_i
\left(
M_i\|(\bh_i - \widetilde  \bh_i )^\top\|_2
+
\widetilde M_i \|\widetilde \bh_i^\top\|_2
\right)^2 . 
\] 
In particular, if the classification head is not pruned, so that 
$\widetilde A=A$, then $\widetilde M_i =0$ and 
\[
\frac{1}{n}\sum_{i=1}^n
\mathrm{KL}(\bp_i \,\|\, \widetilde \bp_i)
\le
\frac{1}{2n}
\sum_{i=1}^n \zeta_i
M_i^2 \|(\bh_i - \widetilde \bh_i)^\top\|_2^2.
\]

\end{theorem}

\begin{proof}

We start by showing that $\nabla\psi=\operatorname{softmax}$.\\
Let $\bu\in\RR^{K}$. 
By the chain rule, for each $j\in[K]$ we have
\[
\frac{\partial \psi}{\partial u_j}(\bu)=
\frac{1}{\sum_{k=1}^{K} e^{u_k}} \cdot \frac{\partial}{\partial u_j}
   \!\left(\sum_{k=1}^{K} e^{u_k}\right)
= \frac{e^{u_j}}{\sum_{k=1}^{K} e^{u_k}}.
\]
which is exactly $\operatorname{softmax}(\bu)_j$. 
Recall that the KL divergence between distributions $\bp$ and $\bq$ over $[K]$ is defined as 
\[
\mathrm{KL}(\bp \,\|\, \bq)=\sum_{k\in[K]} p_k \log\frac{p_k}{q_k}
\]
and the Bregman divergence of a differentiable function $f$ is defined as \[
D_{f}(x,y)=f(x)-f(y)-\inner{\nabla f(y),x-y}
\]
Let $\bu,\bv\in\RR^K$. We define 
\[
\bp=\nabla\psi(\bu),
\qquad
\bq=\nabla\psi(\bv),
\]
the softmax distributions. 
We aim to prove that 
\[
\mathrm{KL}(\bp \,\|\, \bq) = D_{\psi}(\bv,\bu)
\]
Indeed, since $\nabla \psi(\bu)=\operatorname{softmax}( \bu)=\bp$, we have $p_k = e^{u_k}/\sum_{j\in [K]} e^{u_j}$. Taking logarithms,
\[
\log p_k \;=\; \log \frac{e^{u_k}}{\sum_{j\in [K]} e^{u_j}} \;=\; u_k - \log\sum_{j\in [K]} e^{u_j}
\;=\; u_k - \psi(\bu).
\]
Likewise,
\[
\log q_k \;=\; v_k - \psi(\bv).
\]
We have, 
\begin{align*}
\mathrm{KL}(\bp\,\|\,\bq)
&\;=\; \sum_{k=1}^{K} p_k\bigl(\log \frac{p_k}{q_k}\bigr) 
\\
&\;=\; \sum_{k=1}^{K} p_k\bigl(\log p_k - \log q_k\bigr) \\
&\;=\; \sum_{k=1}^{K} p_k
        \Bigl[\bigl(u_k - \psi(\bu)\bigr) - \bigl(v_k - \psi(\bv)\bigr)\Bigr] \\
&\;=\; \sum_{k=1}^{K} p_k\,(u_k - v_k)
       \;+\; \bigl(\psi(\bv) - \psi(\bu)\bigr)\sum_{k=1}^{K} p_k.
\end{align*}
Since $p$ is a distribution on $[K]$, we have $\sum_{k=1}^{K} p_k = 1$. Thus
\begin{align*}
\mathrm{KL}(\bp\,\|\,\bq)
&\;=\; \sum_{k=1}^{K} p_k\,(u_k - v_k) \;+\; \psi(\bv) - \psi(\bu)  
\\
&\;=\; \psi(\bv) \;-\; \psi(\bu) \;-\; \sum_{k=1}^{K} p_k\,(v_k - u_k)
\end{align*}
By definition $\bp=\nabla \psi(\bu)$, hence
\[
\sum_{k=1}^{K} p_k\,(v_k - u_k)
\;=\; \inner{\nabla \psi(\bu),\bv-\bu}.
\]
Substituting, we finally get 
\begin{align*}
\mathrm{KL}(\bp\,\|\,\bq)
&\;=\;  \psi(\bv) - \psi(\bu) \;-\; \inner{\nabla \psi(\bu),\bv-\bu}
\\
&\;=\; D_\psi(\bv,\bu)
\end{align*}
We substitute $\bz_i$ for $\bu$ and $\widetilde\bz_i$ for $\bv$ and since $\nabla \psi(\bz_i)=\bp_i$ and $\nabla \psi(\widetilde \bz_i)=\widetilde \bp_i$ we get 
\begin{align}\label{eq:KL_bregman}
\mathrm{KL}(\bp_i\,\|\,\widetilde \bp_i) 
\;=\; D_\psi(\widetilde \bz_i,\bz_i)
\end{align} 
The multivariate version of Taylor's theorem tells us that for a twice continuously-differentiable function $f:\RR^K\to\RR$ at point $\by\in\RR^K$ we have $\forall \bh \in \RR^K$ 
\[
f(\by+\bh)
=
f(\by)
+
\langle \nabla f(\by),\bh\rangle
+
\int_0^1 (1-t)\,
\bh^\top \nabla^2 f(\by+t\bh)\bh\,dt.
\]
We apply this formula to the function $\psi$, since it is twice continuously-differentiable, with
\[
\by=\bz_i,
\qquad
\bh=\Delta \bz_i.
\]
Thus $\by+\bh=\widetilde \bz_i$, and we have 
\[
\psi(\widetilde \bz_i)
=
\psi(\bz_i)
+
\left\langle \nabla\psi(\bz_i),\Delta \bz_i\right\rangle
+
\int_0^1 (1-t)\,
\Delta \bz_i^\top
\nabla^2\psi(\bz_i+t\Delta \bz_i)
\Delta \bz_i
\,dt.
\]
Rearranging gives
\[
\psi(\widetilde \bz_i)
-
\psi(\bz_i)
-
\left\langle \nabla\psi(\bz_i),\Delta \bz_i\right\rangle
=
\int_0^1 (1-t)\,
\Delta \bz_i^\top
\nabla^2\psi(\bz_i+t\Delta \bz_i)
\Delta \bz_i
\,dt.
\]
From what we showed in \eqref{eq:KL_bregman}, we get
\begin{align}\label{eq:KL_integral}
\mathrm{KL}(\bp_i\,\|\,\widetilde \bp_i)
\;=\;
\psi(\widetilde \bz_i) - \psi(\bz_i) \;-\; \inner{\nabla \psi(\bz_i),\Delta \bz_i} \\
\;=\;
\int_0^1 (1-t)\,
\Delta \bz_i^\top
\nabla^2\psi(\bz_i+t\Delta \bz_i)
\Delta \bz_i
\,dt.
\end{align}
By the definition of the supremum, for every $t\in[0,1]$ we have
\[
\frac{\Delta \bz_i^\top\nabla^2\psi(\bz_i+t\Delta \bz_i)\Delta \bz_i}{\|\Delta \bz_i\|_2^2}\leq \zeta_i.
\]
Multiplying both sides by $\|\Delta \bz_i\|_2^2$ gives
\[
\Delta \bz_i^\top\nabla^2\psi(\bz_i+t\Delta \bz_i)\Delta \bz_i\leq\zeta_i\|\Delta \bz_i\|_2^2.
\]
Therefore,
\begin{align*}
&\int_0^1(1-t)\,\Delta \bz_i^\top\nabla^2\psi(\bz_i+t\Delta \bz_i)\Delta \bz_i\,dt 
\\&\qquad\leq\int_0^1(1-t)\,\zeta_i\|\Delta \bz_i\|_2^2\,dt 
\\&\qquad=\zeta_i\|\Delta \bz_i\|_2^2\int_0^1(1-t)\,dt \\&\qquad=\frac{\zeta_i}{2}\|\Delta \bz_i\|_2^2.
\end{align*}
In particular, if we use the fact that the softmax function is $\frac{1}{2}$ Lipschitz \cite{nair2025softmax}, we get that its Jacobian is bounded by $\frac{1}{2}$. But its Jacobian is the Hessian of the log-sum-exp function, thus for every $t\in[0,1]$,
\[
\Delta \bz_i^\top\nabla^2\psi(\bz_i+t\Delta \bz_i)\Delta \bz_i\leq\frac{1}{2}\|\Delta \bz_i\|_2^2.
\]
Hence
\[
\zeta_i\leq \frac{1}{2},
\]
and consequently
\[\int_0^1(1-t)\,\Delta \bz_i^\top\nabla^2\psi(\bz_i+t\Delta \bz_i)\Delta \bz_i\,dt\leq\frac{1}{4}\|\Delta \bz_i\|_2^2.
\]
From \eqref{eq:KL_integral} we get 
\begin{align}\label{eq:kl_zeta}
\mathrm{KL}(\bp_i\,\|\,\widetilde \bp_i) 
\;\le\;
\frac{1}{2}\zeta_i \|\Delta \bz_i\|_2^2
\;\leq\;
\frac{1}{4}\|\Delta \bz_i\|_2^2.
\end{align}
It remains to relate the logit error to representation and head recovery. Since 
\[
\Delta z_i
\;=\;
\widetilde \bh_i^\top \widetilde A \;-\; \bh_i^\top A
\;=\;
(\widetilde \bh_i-\bh_i)^\top A
\;+\;
\widetilde \bh_i^\top (\widetilde A-A) ,
\]
the triangle inequality gives 
\[
\|\Delta \bz_i\|_2
\le
\|(\widetilde \bh_i-\bh_i)^\top A\|_2
+
\|\widetilde \bh_i^\top (\widetilde A-A)\|_2 .
\]
We note that we can use the sub-multiplicativity of the norm and get 
\[
\|\Delta \bz_i\|_2
\le
\| A\|_2\|(\widetilde \bh_i-\bh_i)\|_2
+
\|(\widetilde A-A)\|_2 \|\widetilde \bh_i\|_2 .
\]
But using the definitions in \eqref{eq:defs} we can get 
\begin{align}
\|\Delta \bz_i\|_2
\le
M_i \|(\bh_i-\widetilde \bh_i)\|_2
+
\widetilde M_i \|\widetilde \bh_i\|_2.
\end{align}
Substituting into \eqref{eq:kl_zeta} and averaging over the calibration examples concludes the proof.
\end{proof}

We note that using the two global bounds introduced in the proof gives us instead the following bound:
\[
\frac{1}{n}\sum_{i=1}^n
\mathrm{KL}(\bp_i \,\|\, \widetilde \bp_i)
\le
\frac{1}{4n}
\sum_{i=1}^n
\left(
\|A\|_2\|(\bh_i - \widetilde  \bh_i )^\top\|_2
+
\|\widetilde A-A\|_2 \|\widetilde \bh_i^\top\|_2
\right)^2 . 
\]
The global bound is valid, but can be overly conservative in practice. 
Indeed, the proof above keeps two example-dependent quantities: $M_i$, which measures the gain of the classification head in the realized representation-error direction, and $\zeta_i$, which measures the curvature of the log-sum-exp function along the corresponding path. 
Both quantities admit worst-case global upper bounds, namely $M_i \leq \|A\|_2$ and $\zeta_i \leq \frac{1}{2}$. 
Substituting these worst-case constants yields the global bound shown in \Cref{fig:four_bars}. 
Although this bound still upper-bounds the empirical KL divergence, it is much looser than the local, batch-dependent bound.

\Cref{fig:zeta_mi_bound} visualizes the empirical distributions of $M_i$ and $\zeta_i$ over the calibration examples and compares them to their corresponding global upper bounds. 
The figure shows that the realized values are substantially smaller than their worst-case bounds: the head gain $M_i$ is typically well below $\|A\|_2$, and the local curvature $\zeta_i$ is far below the worst-case smoothness constant $\frac{1}{2}$.

In Figures \ref{fig:theory_empirical_kl}, \ref{fig:four_bars}, and \ref{fig:zeta_mi_bound}, we empirically approximate $\zeta_i$ by taking the maximum over $\frac{\Delta \bz_i^\top\nabla^2\psi(\bz_i+t\Delta \bz_i)\Delta \bz_i}{\|\Delta \bz_i\|_2^2}$ for $20$ points in $[0,1]$.

\begin{figure}[H]
\centering
\begin{tikzpicture}[trim axis left, trim axis right]
\begin{axis}[
    ybar=0.05,
    ymode=log,
    ymin=1e-2,
    log origin=infty, 
    ylabel={Value (log scale)},
    width=0.8\linewidth,
    height=5cm,
    bar width=30pt,
    xtick={1,2,3,4},
    xticklabels={
    {Actual},
    {Per example},
    {Head-norm},
    {Global},
    },
    xticklabel style={
        align=center,
    },
    enlarge x limits=0.35,
    legend style={
        at={(0.5,-0.25)},
        anchor=north,
        legend columns=2
    },
    scale only axis,
    nodes near coords={
        \pgfmathprintnumber[fixed,precision=2]{\pgfplotspointmeta}
    },
    point meta=rawy,
    every node near coord/.append style={
        anchor=north,
        yshift=2pt, color=black, opacity=1
    },
    clip=false
]

\addlegendimage{ybar, fill=gray, draw=gray, fill opacity=0.35, draw opacity=0.35}
\addlegendentry{Before Recovery}
\addlegendimage{ybar, fill=gray, draw=gray}
\addlegendentry{After Recovery}

\addplot+[bar shift=-15pt, fill=myblue, draw=myblue, fill opacity=0.35, draw opacity=0.35]
coordinates {(1,5.7127e+00)};

\addplot+[bar shift=15pt, fill=myblue, draw=myblue]
coordinates {(1,1.3363e-01)};

\addplot+[bar shift=-15pt, fill=mygreen, draw=mygreen, fill opacity=0.35, draw opacity=0.35]
coordinates {(2,9.0598e+00)};

\addplot+[bar shift=15pt, fill=mygreen, draw=mygreen]
coordinates {(2,1.9148e-01)};

\addplot+[bar shift=-15pt, fill=myred, draw=myred, fill opacity=0.35, draw opacity=0.35]
coordinates {(3,6.5355e+01)};

\addplot+[bar shift=15pt, fill=myred, draw=myred]
coordinates {(3,1.5744e+00)};

\addplot+[bar shift=-15pt, fill=mypurple, draw=mypurple, fill opacity=0.35, draw opacity=0.35]
coordinates {(4,3.2865e+03)};

\addplot+[bar shift=15pt, fill=mypurple, draw=mypurple]
coordinates {(4,3.9926e+02)};

\node[
    anchor=north,
    align=left,
    font=\small,
    text width=0.95\linewidth
] at (axis description cs:0.5,-0.33) {
\begin{tabular}{@{}ll@{}}
\textcolor{myblue}{\rule{1.2em}{0.8em}} &
Actual: $\displaystyle \frac{1}{n}\sum_{i=1}^n \mathrm{KL}(p_i \,\|\, \widetilde p_i)$ \\[0.6em]

\textcolor{mygreen}{\rule{1.2em}{0.8em}} &
Per-example bound: $\displaystyle \frac{1}{2n}
\sum_{i=1}^n \zeta_i M_i^2
\|\bh_i - \widetilde{\bh}_i \|_2^2$ \\[0.6em]

\textcolor{myred}{\rule{1.2em}{0.8em}} &
Head-norm bound: $\displaystyle \frac{\|A\|_2^2}{2n}
\sum_{i=1}^n \zeta_i
\|\bh_i - \widetilde{\bh}_i \|_2^2$ \\[0.6em]

\textcolor{mypurple}{\rule{1.2em}{0.8em}} &
Global bound: $\displaystyle \frac{\|A\|_2^2}{4n}
\sum_{i=1}^n
\|\bh_i - \widetilde{\bh}_i \|_2^2$
\end{tabular}
};

\end{axis}
\end{tikzpicture}
\caption{
    Empirical verification of the KL divergence bound. Comparing both the global bounds as well as the data-dependent bounds. It can be seen that STARFISH recovery improves all bounds, and that the per-example one remains the tightest. 
    }   \label{fig:four_bars}
\end{figure}
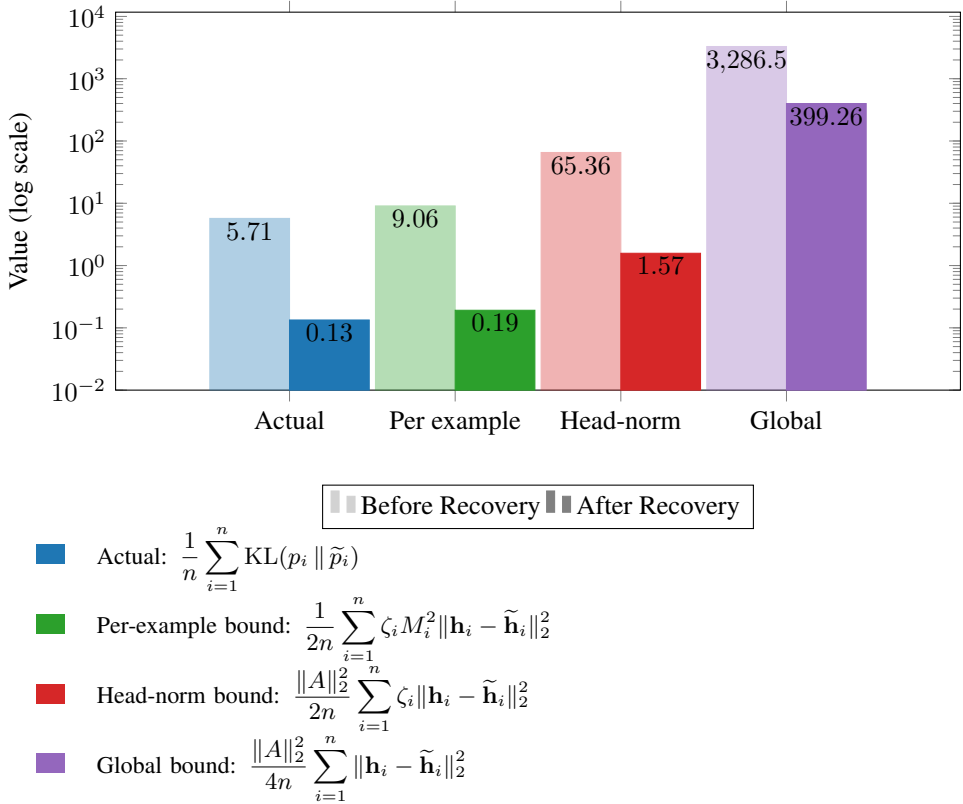

\begin{figure}[H]
    \centering
    \includegraphics[width=0.85\textwidth]{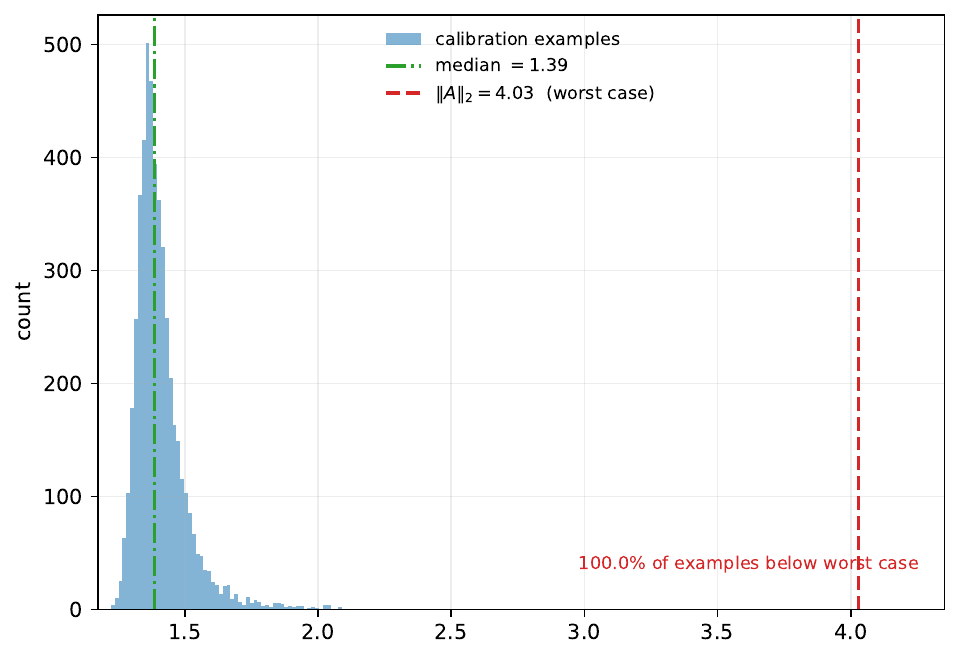} 
    
    \vspace{0.5cm} 
    
    \includegraphics[width=0.85\textwidth]{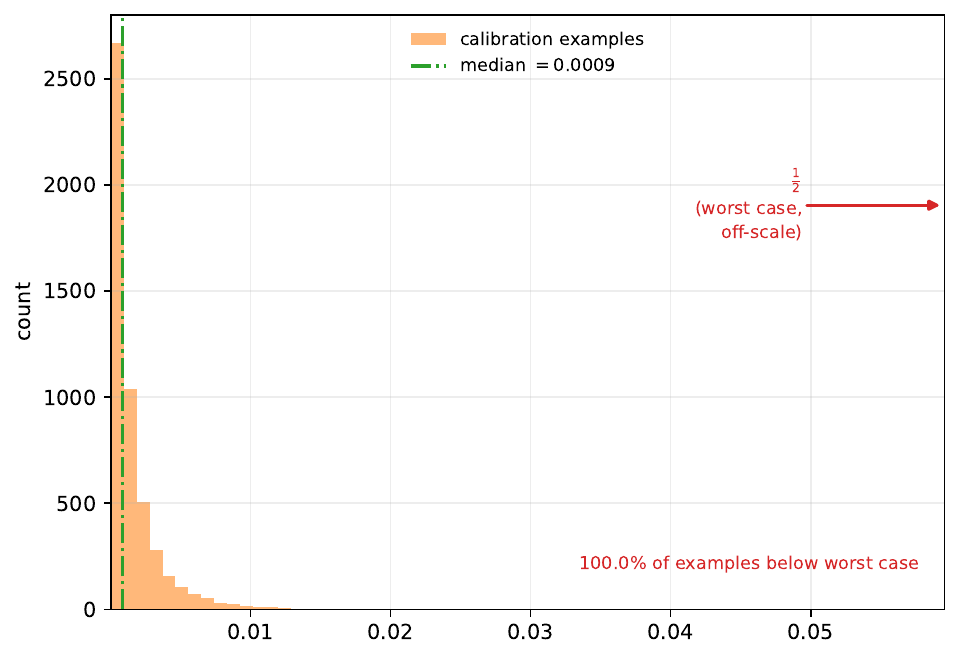} 
    
    \caption{Visualization of the data-dependent constants in the local KL bound compared to their global worst-case upper bounds.
    \textbf{Top}: Distribution of the example-wise head gains $M_i$ over the calibration examples, compared to the global bound $\|A\|_2$.
    \textbf{Bottom}: Distribution of the example-wise curvature constants $\zeta_i$, compared to the global smoothness bound $\frac{1}{2}$.
    In both cases, the realized data-dependent quantities are substantially smaller than their worst-case bounds, explaining why the local batch-dependent KL bound in \Cref{fig:four_bars} is much tighter than the corresponding global bound.
}
    \label{fig:zeta_mi_bound}
\end{figure}

\subsection{Recovery vs. fine-tuning: further experiment and details}\label[appendix]{app:recovery_vs_fine}

In \Cref{sec:experiments:loss_ablation}, we compared STARFISH to standard output-level recovery objectives on DeiT-B under unstructured pruning. Here, we provide an additional comparison to CE in structured pruning of various DeiT models and pruning scopes. 

\Cref{fig:lizard_deit_sparse_fine} compares STARFISH to CE (i.e., standard fine-tuning) after applying CORP structured pruning to several DeiT architectures and pruning scopes. Across all models and sparsity levels, STARFISH outperforms CE. The gap is especially pronounced at the hardest settings, the joint QK+MLP setting in the two leftmost graphs, where the label-based CE provides only a weak recovery signal. 

The CE and KL baselines reported in \Cref{tab:deitt_loss_ablation} and \Cref{fig:lizard_deit_sparse_fine} use images from the ImageNet training set, and have all other parameters such as batch size, epochs, learning rate etc. set exactly as the STARFISH recovery method, which are disclosed in \Cref{app:experimentaldetails}. The experiment reported in \Cref{tab:deitt_loss_ablation} reports recovery after unstructured pruning on the QKV, Out, and MLP parameters of a DeiT-B model. The STARFISH+KL objective combines the two losses with equal weights of $0.5$ and uses a calibration set drawn from the test set.

\begin{figure}[H]
\centering
\begin{subfigure}[t]{0.3\linewidth}
\centering
\begin{tikzpicture}
\begin{axis}[
    title={DeiT3-H (QK+MLP)},
    title style={font=\scriptsize},
    ylabel={Top-1 Accuracy (\%)},
    xlabel={Sparsity Level},
    xmin=0.49, xmax=0.86,
    ymax=87, ymin=0, yticklabel style={font=\scriptsize}, 
    xticklabel style={rotate=45, anchor=east, font=\scriptsize},
    xtick={0.5,0.63,0.69,0.75,0.8,0.85},
    legend to name=commonlegend,
    legend columns=5,
    legend style={draw=none, /tikz/every even column/.append style={column sep=0.5em}},
    ymajorgrids=true,
    grid style=dashed,
    width=\linewidth,
    height=5.5cm,
    thick,
    font=\scriptsize
]
\addplot[color=black, dashed, thick] coordinates {
    (0.49,84.97) (0.86,84.97)
};
\addlegendentry{Dense acc.}

\addplot[color=gray, mark=square, dashed] coordinates {
    (0.5,78.27) (0.63,64.38) (0.69,46.72) (0.75,25.90) (0.8,13.19) (0.85,5.88)
};
\addlegendentry{Pruned acc.}

\addplot[color=MidnightBlue, mark=triangle] coordinates {
    (0.5,66.78) (0.63,65.19) (0.69,64.12) (0.75,61.54) (0.8,58.73) (0.85,54.75)
};
\addlegendentry{CE}

\addplot[color=Salmon, mark=*, ultra thick] coordinates {
    (0.5,84.06) (0.63,83.35) (0.69,82.79) (0.75,81.78) (0.8,80.34) (0.85,78.24)
};
\addlegendentry{\textbf{STARFISH}}

\end{axis}
\end{tikzpicture}
\label{fig:lizard_deit_h_qkv_mlp_fine}
\end{subfigure}%
\hspace{-3em}
\begin{subfigure}[t]{0.3\linewidth}
\centering
\begin{tikzpicture}
\begin{axis}[
    title={DeiT-B (QK+MLP)},
    title style={font=\scriptsize},
    xlabel={Sparsity Level},
    xmin=0.49, xmax=0.86,
    ymax=87, ymin=0,
    xticklabel style={rotate=45, anchor=east, font=\scriptsize},
    xtick={0.5,0.63,0.69,0.75,0.8,0.85},
    ymajorgrids=true,
    grid style=dashed,
    width=\linewidth,
    height=5.5cm,
    thick,
    yticklabels=\empty,
    font=\scriptsize
]
\addplot[color=black, dashed, thick] coordinates {
    (0.49,81.73) (0.86,81.73)
};

\addplot[color=gray, mark=square, dashed] coordinates {
    (0.5,31.19) (0.63,8.30) (0.69,3.54) (0.75,1.42) (0.8,0.95) (0.85,0.62)
};

\addplot[color=MidnightBlue, mark=triangle] coordinates {
    (0.5,60.47) (0.63,56.16) (0.69,52.86) (0.75,47.87) (0.8,42.43) (0.85,35.40)
};

\addplot[color=Salmon, mark=*, ultra thick] coordinates {
    (0.5,75.77) (0.63,72.29) (0.69,70.12) (0.75,67.03) (0.8,64.03) (0.85,59.70)
};
\end{axis}

\end{tikzpicture}
\label{fig:lizard_deit_b_qkv_mlp_fine}
\end{subfigure}%
\hspace{-3.8em}
\begin{subfigure}[t]{0.3\linewidth}
\centering
\begin{tikzpicture}
\begin{axis}[
    title={DeiT-B (MLP)},
    title style={font=\scriptsize},
    xlabel={Sparsity Level},
    xmin=0.29, xmax=0.81,
    ymax=87, ymin=0,
    yticklabel style={font=\scriptsize},
    yticklabels=\empty,
    xticklabel style={rotate=45, anchor=east, font=\scriptsize},
    xtick={0.3,0.5,0.7,0.75,0.8},
    ymajorgrids=true,
    grid style=dashed,
    width=\linewidth,
    height=5.5cm,
    thick,
    font=\scriptsize
]
\addplot[color=black, dashed, thick] coordinates {
    (0.29,81.73) (0.81,81.73)
};

\addplot[color=gray, mark=square, dashed] coordinates {
    (0.3,71.88) (0.5,54.56) (0.7,14.50) (0.75,7.61) (0.8,4.06)
};

\addplot[color=MidnightBlue, mark=triangle] coordinates {
    (0.3,68.50) (0.5,67.06) (0.7,63.00) (0.75,61.11) (0.8,58.50)
};

\addplot[color=Salmon, mark=*, ultra thick] coordinates {
    (0.3,78.98) (0.5,75.91) (0.7,69.74) (0.75,67.15) (0.8,64.28)
};
\end{axis}
\end{tikzpicture}
\label{fig:lizard_deit_b_mlp_fine}
\end{subfigure}
\hspace{-3.8em}
\begin{subfigure}[t]{0.3\linewidth}
\centering
\begin{tikzpicture}
\begin{axis}[
    title={DeiT-S (MLP)},
    title style={font=\scriptsize},
    xlabel={Sparsity Level},
    xmin=0.29, xmax=0.81,
    ymax=87, ymin=0, yticklabel style={font=\scriptsize},   yticklabels=\empty,
    xticklabel style={rotate=45, anchor=east, font=\scriptsize},
    xtick={0.3,0.5,0.7,0.75,0.8},
    ymajorgrids=true,
    grid style=dashed,
    width=\linewidth,
    height=5.5cm,
    thick,
    font=\scriptsize
]
\addplot[color=black, dashed, thick] coordinates {
    (0.29,79.72) (0.81,79.72)
};

\addplot[color=MidnightBlue, mark=triangle] coordinates {
    (0.3,67.10) (0.5,62.82) (0.7,52.91) (0.75,47.95) (0.8,42.25)
};

\addplot[color=gray, mark=square, dashed] coordinates {
    (0.3,54.11) (0.5,45.51) (0.7,10.80) (0.75,3.84) (0.8,1.95)
};

\addplot[color=Salmon, mark=*, ultra thick] coordinates {
    (0.3,76.06) (0.5,72.54) (0.7,65.16) (0.75,61.68) (0.8,57.40)
};
\end{axis}
\end{tikzpicture}
\label{fig:lizard_deit_s_mlp_fine}
\end{subfigure}%

\centering
\ref*{commonlegend}

\caption{Comparison of the top-1 accuracy results after the STARFISH and fine-tuning healing processes, on different structured pruning scopes and sparsity levels for the DeiT-B, DeiT3-H, and DeiT-S architectures trained on ImageNet, performed following CORP pruning. One can see the STARFISH method achieves high results in all settings, especially in high-sparsity levels.} \label{fig:lizard_deit_sparse_fine}
\end{figure}
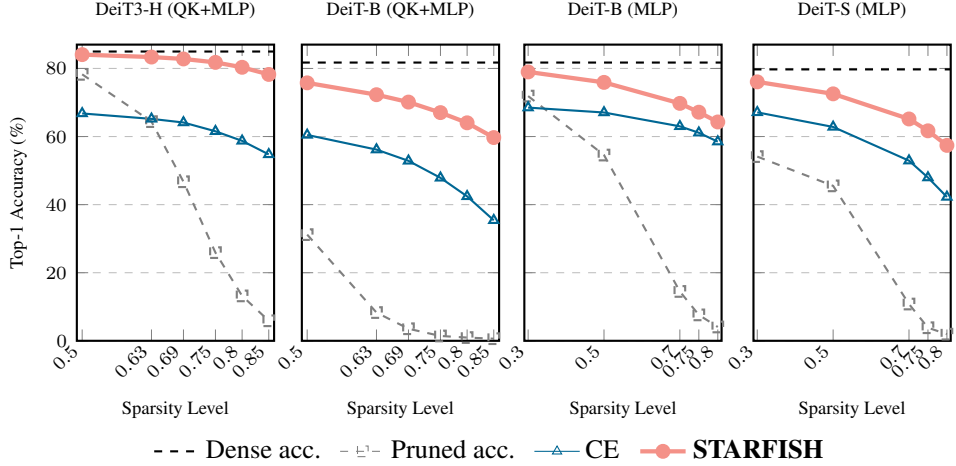

\section{Experimental details}\label[appendix]{app:experimentaldetails}

Extended experimental details for \Cref{sec:experiments}. 
\paragraph{Data splits.}
Unless otherwise stated, all calibration images used for STARFISH recovery are sampled from the unlabeled ImageNet-1K test split, streamed from Hugging Face.\footnote{\url{https://huggingface.co/datasets/ILSVRC/imagenet-1k}\label{fn:huggingface}}
No labels are used during STARFISH recovery. Evaluation is performed on the ImageNet-1K validation split, also streamed from Hugging Face.\footref{fn:huggingface} 
Thus, the calibration images used for representation alignment are disjoint from the labeled validation images used for reporting top-1 accuracy. 

\paragraph{Model checkpoints.}
All dense model checkpoints are loaded from publicly available pretrained weights. For ViT/B-16 and ViT/L-16, we use the torchvision checkpoints
\texttt{ViT\_B\_16\_Weights.IMAGENET1K\_V1} and
\texttt{ViT\_L\_16\_Weights.IMAGENET1K\_SWAG\_LINEAR\_V1}, respectively.\footnote{\url{https://docs.pytorch.org/vision/stable/models/generated/torchvision.models.vit_b_16.html}}\footnote{\url{https://docs.pytorch.org/vision/stable/models/generated/torchvision.models.vit_l_16.html}}
For DeiT and DeiT3 models, we use the pretrained checkpoints provided through \texttt{timm} by calling \texttt{timm.create\_model(model\_name, pretrained=True)}.\footnote{\url{https://huggingface.co/docs/timm/reference/models}}
Specifically, we use the non-distilled DeiT variants
\texttt{deit\_tiny\_patch16\_224},
\texttt{deit\_small\_patch16\_224}, and
\texttt{deit\_base\_patch16\_224}, and the DeiT3 variants
\texttt{deit3\_large\_patch16\_224}, and
\texttt{deit3\_huge\_patch14\_224}. For MobileNetV1, we use the dense STR ~\cite{kusupati2020soft} checkpoint provided by the WoodFisher repository.\footnote{\url{https://github.com/IST-DASLab/WoodFisher/tree/main/checkpoints}}

\paragraph{Trainable parameters during recovery.}
During STARFISH recovery, the pruning mask is kept fixed and we update only the surviving weights in the tensors that were designated as prunable by the corresponding pruning method. All parameters outside the prunable set are frozen at their dense-model values, and all masked-out weights remain fixed at zero throughout recovery. Equivalently, if $\mathcal P$ denotes the set of prunable parameter tensors and $\mathbf m$ is the binary pruning mask, then optimization is performed only over entries $\theta_j$ such that $\theta_j \in \mathcal P$ and $m_j=1$. Entries with $m_j=0$ are not updated, and parameters not belonging to $\mathcal P$ are excluded from the optimizer. In implementation, we enforce this constraint by applying the mask after each optimizer step, ensuring that the sparsity pattern is unchanged throughout recovery. 

\paragraph{Representation specifications.} 
For ViTs, that is the post-block token matrix -- taken after the attention and MLP sublayers, including their residual additions. For MobileNetV1, it is the output of the depthwise-separable convolution block,
taken after the point-wise convolution, BatchNorm, and ReLU. 

\paragraph{Recovery optimization.}
For DeiT-T, DeiT-S, DeiT-B, ViT/B-16, and MobileNetV1, STARFISH recovery is run for $10$ epochs. For DeiT3-L, DeiT3-H, and ViT/L-16, STARFISH is run for $5$ epochs. 
Recovery uses AdamW with initial learning rate $6 \times 10^{-4}$, decayed to $10^{-6}$ with a cosine learning-rate schedule. All results reported in the paper are averaged over three independent random seeds. Across seeds, randomness comes from the sampled calibration subset and the ordering of calibration batches. If not otherwise stated, we used $1{,}000$ examples for DeiT-T, $2{,}500$ examples for DeiT-S, $5{,}000$ examples for DeiT-B
and ViT/B-16, $7{,}000$ examples for DeiT3-L and ViT/L-16 and $9{,}000$ examples for DeiT3-H. We used a batch size of $32$ for all models except MobilenetV1, for which we used a batch size of $64$. 

\paragraph{CORP masks.}
For comparisons with CORP, we use the CORP recovery results as reported in the original CORP paper. To evaluate STARFISH under the same pruning setting, we need CORP-style pruning masks. Since the official CORP code was not publicly available at the time of our experiments, we implemented only the CORP mask-selection procedure and used it to generate masks for STARFISH recovery. The mask-selection calibration set has the same size as the calibration set used later for STARFISH recovery, but is sampled from the ImageNet-1K training split streamed from Hugging Face.\footref{fn:huggingface} After mask selection, STARFISH recovery is performed using the unlabeled ImageNet-1K test-split calibration images described above. The pruned accuracies of our CORP-mask implementation and the STARFISH recovery results are averaged over three independent random seeds.

\paragraph{ImageNet-mini evaluation for SNOWS comparisons.}
All SNOWS results reported in prior work are evaluated on ImageNet-1K-mini rather than on the full ImageNet-1K validation split. Therefore, for the comparisons against SNOWS in Panel B of Table~\ref{tab:structured}, we also evaluate STARFISH on ImageNet-1K-mini, using the same evaluation setting in order to ensure a fair comparison. All other ImageNet-1K results in the paper are evaluated on the full ImageNet-1K validation split unless explicitly stated otherwise.

\subsection{Ablation study: cosine-similarity is better than MSE}\label[appendix]{app:ablation_mse}

STARFISH uses cosine similarity to align the dense and pruned intermediate representations. 
A natural alternative is to replace the cosine loss with a mean-squared-error (MSE) loss between the same representations. 
In \Cref{tab:mse_ablation}, we compare these two objectives for recovering a pruned DeiT-B model pruned under unstructured sparsity on QKV+Out+MLP. Empirically, cosine alignment works better than MSE alignment. The difference is most pronounced at $0.8$-sparsity, with a difference in final accuracy of $2.7$ points. Accordingly, we use cosine alignment as the default STARFISH objective throughout the paper.
\begin{table}[H]
\centering
\caption{Comparison of top-1 accuracy results of using cosine similarity vs. MSE as an objective for recovering a pruned DeiT-B at different sparsity levels. We show the cosine similarity loss presents better recovery results}
\label{tab:mse_ablation}
\begin{tabular}{@{}l cccc @{}}
\toprule
& \multicolumn{4}{c}{\textbf{Dense Acc.: $81.73\%$}} \\
\cmidrule(l){2-5}
\textbf{Method/Sparsity} & \textbf{0.5} & \textbf{0.6} & \textbf{0.7} & \textbf{0.8} \\
\midrule
Pruned Acc. & 74.59  & 56.16  & 20.02  & 0.83 \\
\midrule
MSE & 81.15  & 80.18  & 78.17  & 70.56  \\
Cosine Similarity & \textbf{81.34} & \textbf{ 80.59} & \textbf{78.91} & \textbf{73.26} \\
\bottomrule
\end{tabular}
\end{table}

\subsection{Ablation study: calibration data source}\label[appendix]{app:calib_data_train}

STARFISH only requires unlabeled calibration inputs, and does not rely on access to the original training set. 
To verify that the recovery is not limited to the source of the calibration data, we compare using calibration images sampled from the training set versus the test set, while keeping the calibration size and recovery procedure fixed.

\Cref{tab:calib_train_vs_test} reports the results for both structured DeiT pruning, pruned using CORP, and $2{:}4$ semi-structured ViT pruning, pruned using MP, at $0.5$ sparsity. 
The recovered accuracies obtained from train and test calibration images are very similar, with slightly better results for using the test set as the calibration source.    

\begin{table}[H]
\centering
\caption{
Top-1 accuracy after STARFISH recovery when the unlabeled calibration examples are sampled from the training set versus a held-out test set at $0.5$ sparsity.
\textbf{Panel A:} Structured pruning results on DeiT models.
\textbf{Panel B:} $2{:}4$ semi-structured pruning results on ViT models.
Best result within each model is bolded. Showing the calibration set from the unlabeled test set improves recovery.
}\label{tab:calib_train_vs_test}
{
\begin{tabular}{@{}lcccc@{}}
\toprule
\makecell[l]{Model} 
& \makecell{Pruning\\Scope}  
& \makecell{Calib. Set\\Source} 
& \makecell{Calib.\\Size} 
& \makecell{Recovered\\Acc.} \\

\midrule
\multicolumn{5}{@{}l}{\textbf{Panel A: DeiT setting ($0.5$ structured sparsity)}}\\
\midrule

\multirow{2}{*}{\shortstack[l]{DeiT-T \\(Dense Acc. 72.01)}}
& \multirow{2}{*}{\shortstack{MLP+QK\\(Pruned Acc. 7.57)}}
  &  Train & \multirow{2}{*}{1,000}  & 58.75 \\
& &  Test &  &\textbf{59.01} \\
\midrule

\multirow{2}{*}{\shortstack[l]{DeiT-S \\(Dense Acc. 79.72)}}
& \multirow{2}{*}{\shortstack{MLP+QK\\(Pruned Acc. 17.32)}}
  &  Train & \multirow{2}{*}{2,500}   & 71.25 \\
  & & Test &  & \textbf{71.59} \\
\midrule

\multirow{2}{*}{\shortstack[l]{DeiT-B \\ (Dense Acc. 81.73)}}
& \multirow{2}{*}{\shortstack{MLP+QK\\(Pruned Acc. 31.19)}}
  &  Train & \multirow{2}{*}{5,000}  & 74.95 \\
  &  & Test & & \textbf{75.77} \\
\midrule

\multirow{2}{*}{\shortstack[l]{DeiT3-L \\  (Dense Acc. 84.58)}}
& \multirow{2}{*}{\shortstack{MLP+QK\\(Pruned Acc. 72.09)}}
  &  Train & \multirow{2}{*}{7,000}  & 82.11 \\
& & Test & & \textbf{82.69} \\
\midrule

\multirow{2}{*}{\shortstack[l]{DeiT3-H \\ (Dense Acc. 84.97)}}
& \multirow{2}{*}{\shortstack{MLP+QK\\(Pruned Acc. 78.27)}}
  &  Train & \multirow{2}{*}{9,000} & 83.90 \\
& & Test & &\textbf{84.06}\\
\midrule

\multicolumn{5}{@{}l}{\textbf{Panel B: ViT setting ($2{:}4$ semi-structured sparsity)}} \\
\midrule

\multirow{2}{*}{\shortstack[l]{ViT/B-16 \\ (Dense Acc. 80.40)}}
& \multirow{2}{*}{\shortstack{QKV+Out+MLP\\(Pruned Acc. 70.86)}}
  & Train  & \multirow{2}{*}{5,000} & 79.05  \\
& & Test   &  & \textbf{79.07}  \\

\midrule
\multirow{2}{*}{\shortstack[l]{ViT/L-16 \\ (Dense Acc. 84.20)}}
& \multirow{2}{*}{\shortstack{MLP\\(Pruned Acc. 0.20)}}
  & Train  & \multirow{2}{*}{7,000} & 70.88  \\
& & Test   &  & \textbf{72.46}  \\
\bottomrule
\end{tabular}
}
\end{table}

\subsection{Runtime}\label[appendix]{app:runtime}

STARFISH provides a direct runtime--accuracy tradeoff when using an even smaller calibration set. 
In \Cref{tab:runtime_deit3}, we recover DeiT3-L and DeiT3-H models pruned to $0.8$-sparsity using CORP structured masks on the QK+MLP pruning scope, and vary only the number of calibration examples used by STARFISH. It is no surprise that, for smaller calibration sets, runtime scales approximately linearly with the calibration set size, since the number of recovery epochs and batch size are fixed across all runs. 

Even very small calibration sets already provide substantial recovery. 
For example, on DeiT3-H, STARFISH with only 1,000 calibration examples runs in about 16 minutes and recovers almost $86\%$ of the original dense model accuracy. Using larger calibration sets further improves recovery: with 9,000 examples, STARFISH reaches $94.55\%$ accuracy retention in the same DeiT3-H setting. 
This demonstrates that STARFISH can be used either as a very fast recovery method with modest calibration data, or as a stronger recovery method when more calibration time is available. The runtime results were measured using a single NVIDIA A10 GPU with 24GB memory. 

\begin{table}[t]
\centering
\caption{
Runtime and final top-1 accuracy for STARFISH recovery on DeiT3-L and DeiT3-H under different calibration set sizes. 
Runtime is reported as mean $\pm$ sample standard deviation over three runs. It can be seen that runtime scales almost linearly, and that even using small calibration sets STARFISH recovers to a substantial accuracy.
}
\label{tab:runtime_deit3}
{
\begin{tabular}{@{}lcccc@{}}
\toprule
\makecell[l]{Model} 
& \makecell{Calib.\\Size} 
& \makecell{Runtime} 
& \makecell{Recovered\\Acc.} 
& \makecell{Retention\\(\%)} 
\\
\midrule

\multirow{3}{*}{\shortstack[l]{DeiT3-L \\ (Dense Acc. 84.58)}}
& 500   & 3m 10s $\pm$ 10s & 54.98 & 65.00\% \\ 
& 1,000 & 5m 41s $\pm$ 12s & 63.03 & 74.52\% \\
& 7,000 & 34m 32s $\pm$ 14s & 75.4 & 89.15\% 
\\

\midrule

\multirow{3}{*}{\shortstack[l]{DeiT3-H \\ (Dense Acc. 84.97)}}
& 500   & 8m 24s $\pm$ 2s & 68.63 & 80.77\% \\ 
& 1,000 & 16m 7s $\pm$ 9s & 73.06 & 85.98\% \\
& 9,000 & 2h 14m 57s $\pm$ 22s & 80.34 & 94.55\% 
\\
\bottomrule
\end{tabular}
}
\end{table}

\subsection{Raw data}\label[appendix]{app:raw_data}

For completeness, we report the raw numerical values used in the plots in the main paper. 
The corresponding trends are discussed in the main text; here, the tables are intended only to make the plotted results explicit and reproducible. 
\Cref{tab:mlp_all_sparsities} and \Cref{tab:mlp_qkv_all_sparsities} provide the raw data for the structured DeiT pruning experiments shown in \Cref{fig:lizard_high_sparsity_corp}, while \Cref{tab:mobilev1_raw_sparsity} provides the raw data for the MobileNetV1 sparsity experiment shown in \Cref{fig:lizard_high_sparsity_mobile}.

\begin{table}[H]
\centering
\caption{
Raw MobileNetV1 top-1 accuracy values across sparsity levels for the experiment shown in \Cref{fig:lizard_high_sparsity_mobile}. 
$\Delta$ denotes the change in top-1 accuracy relative to the corresponding dense baseline.
}
\label{tab:mobilev1_raw_sparsity}
\begin{tabular}{@{}lccccccccc@{}}
\toprule
\multirow{2}{*}{Method}
& \multirow{2}{*}{\makecell{Dense\\Acc.}}
& \multicolumn{2}{c}{$0.5$-Sparsity}
& \multicolumn{2}{c}{$0.6$-Sparsity}
& \multicolumn{2}{c}{$0.7$-Sparsity}
& \multicolumn{2}{c}{$0.8$-Sparsity} \\
\cmidrule(lr){3-4}
\cmidrule(lr){5-6}
\cmidrule(lr){7-8}
\cmidrule(l){9-10}
& 
& Acc. & $\Delta$
& Acc. & $\Delta$
& Acc. & $\Delta$
& Acc. & $\Delta$ \\
\midrule

MP 
& 71.95 
& 63.83 & $-8.12$
& 47.15 & $-24.80$
& 11.68 & $-60.27$
& 0.43  & $-71.52$ \\

WF 
& 71.95
& 68.91 & $-3.04$
& 60.90 & $-11.05$
& 29.36 & $-42.59$
& 0.24  & $-71.71$ \\

CBS 
& 71.95
& 70.21 & $-1.74$
& 66.37 & $-5.58$
& 55.11 & $-16.84$
& 16.38 & $-55.57$ \\

CHITA 
& 71.95
& 70.42 & $-1.53$
& 67.30 & $-4.65$
& 59.40 & $-12.55$
& 29.78 & $-42.17$ \\

FALCON 
& 71.95
& 70.35 & $-1.60$
& 67.18 & $-4.77$
& 58.40 & $-13.55$
& 25.82 & $-46.13$ \\

SNOWS 
& 70.89
& 70.10 & $-0.79$
& 68.70 & $-2.19$
& 65.28 & $-5.61$
& 54.83 & $-16.06$ \\

\textbf{STARFISH}
& 71.95
& \textbf{71.12} & $-0.83$
& \textbf{70.10} & $-1.85$
& \textbf{67.67} & $-4.28$
& \textbf{60.66} & $-11.29$ \\

\bottomrule
\end{tabular}
\end{table}

\clearpage

\begin{table}[H]
\centering
\setlength{\tabcolsep}{4pt}
\caption{
Raw top-1 accuracy values for the structured MLP-pruning DeiT experiments shown in the right two panels of \Cref{fig:lizard_high_sparsity_corp}. 
We report the pruned accuracy, recovered accuracy, accuracy retention relative to the dense model, and the retention improvement of STARFISH over CORP, in the settings in which CORP is reported. 
Best reported recovered result within each model and sparsity level is bolded.
}
\label{tab:mlp_all_sparsities}
\begin{tabular}{@{}lccccccc@{}}
\toprule
\makecell[l]{Model}
& \makecell{Sparsity\\Level}
& \makecell{Pruned\\Acc.}
& \makecell{Method}
& \makecell{Calib.\\Size}
& \makecell{Recovered\\Acc.}
& \makecell{Retention\\(\%)}
& \makecell{Improvement\\($\Delta$\%)} \\
\midrule

\multirow{10}{*}{\shortstack[l]{DeiT-S \\(Dense Acc. 79.72)}}
& \multirow{2}{*}{$0.30$}
& \multirow{2}{*}{54.11}
  & CORP   & --    & 73.88 & 92.67\% & \multirow{2}{*}{\sgreen{\textbf{2.73}}} \\
& & & STARFISH & 2,500 & \textbf{76.06} & \textbf{95.41\%} & \\
\cmidrule(l){2-8}

& \multirow{2}{*}{$0.50$}
& \multirow{2}{*}{45.51}
  & CORP   & --    & 65.85 & 82.60\% & \multirow{2}{*}{\sgreen{\textbf{8.39}}} \\
& & & STARFISH & 2,500 & \textbf{72.54} & \textbf{90.99\%} & \\
\cmidrule(l){2-8}

& \multirow{2}{*}{$0.70$}
& \multirow{2}{*}{10.80}
  & CORP   & --    & 49.98 & 62.69\% & \multirow{2}{*}{\sgreen{\textbf{19.04}}} \\
& & & STARFISH & 2,500 & \textbf{65.16} & \textbf{81.74\%} & \\
\cmidrule(l){2-8}

& \multirow{2}{*}{$0.75$}
& \multirow{2}{*}{3.84}
  & CORP   & --    & --    & --       & \multirow{2}{*}{--} \\
& & & STARFISH & 2,500 & \textbf{61.68} & \textbf{77.37\%} & \\
\cmidrule(l){2-8}

& \multirow{2}{*}{$0.80$}
& \multirow{2}{*}{1.95}
  & CORP   & --    & --    & --       & \multirow{2}{*}{--} \\
& & & STARFISH & 2,500 & \textbf{57.40} & \textbf{72.00\%} & \\

\midrule

\multirow{10}{*}{\shortstack[l]{DeiT-B \\(Dense Acc. 81.73)}}
& \multirow{2}{*}{$0.30$}
& \multirow{2}{*}{71.88}
  & CORP   & --    & 75.65 & 92.56\% & \multirow{2}{*}{\sgreen{\textbf{4.07}}} \\
& & & STARFISH & 5,000 & \textbf{78.98} & \textbf{96.64\%} & \\
\cmidrule(l){2-8}

& \multirow{2}{*}{$0.50$}
& \multirow{2}{*}{54.56}
  & CORP   & --    & 68.67 & 84.02\% & \multirow{2}{*}{\sgreen{\textbf{8.86}}} \\
& & & STARFISH & 5,000 & \textbf{75.91} & \textbf{92.88\%} & \\
\cmidrule(l){2-8}

& \multirow{2}{*}{$0.70$}
& \multirow{2}{*}{14.50}
  & CORP   & --    & 56.46 & 69.08\% & \multirow{2}{*}{\sgreen{\textbf{16.25}}} \\
& & & STARFISH & 5,000 & \textbf{69.74} & \textbf{85.33\%} & \\
\cmidrule(l){2-8}

& \multirow{2}{*}{$0.75$}
& \multirow{2}{*}{7.61}
  & CORP   & --    & --    & --       & \multirow{2}{*}{--} \\
& & & STARFISH & 5,000 & \textbf{67.15} & \textbf{82.16\%} & \\
\cmidrule(l){2-8}

& \multirow{2}{*}{$0.80$}
& \multirow{2}{*}{4.06}
  & CORP   & --    & --    & --       & \multirow{2}{*}{--} \\
& & & STARFISH & 5,000 & \textbf{64.28} & \textbf{78.65\%} & \\

\bottomrule
\end{tabular}
\end{table}

\begin{table}[H]
\centering
\setlength{\tabcolsep}{4pt}
\caption{
Raw top-1 accuracy values for the structured QK+MLP-pruning DeiT experiments shown in the left two panels of \Cref{fig:lizard_high_sparsity_corp}. 
We report the pruned accuracy, recovered accuracy, accuracy retention relative to the dense model, and the retention improvement of STARFISH over CORP, in the settings in which CORP is reported. 
Best reported recovered result within each model and sparsity level is bolded.}
\label{tab:mlp_qkv_all_sparsities}
\begin{tabular}{@{}lccccccc@{}}
\toprule
\makecell[l]{Model}
& \makecell{Sparsity\\Level}
& \makecell{Pruned\\Acc.}
& \makecell{Method}
& \makecell{Calib.\\Size}
& \makecell{Recovered\\Acc.}
& \makecell{Retention\\(\%)}
& \makecell{Improvement\\($\Delta$\%)} \\
\midrule

\multirow{12}{*}{\shortstack[l]{DeiT-B \\(Dense Acc. 81.73)}}
& \multirow{2}{*}{$0.50$}
& \multirow{2}{*}{31.19}
  & CORP   & --    & 67.00 & 81.98\% & \multirow{2}{*}{\sgreen{\textbf{10.73}}} \\
& & & STARFISH & 5,000 & \textbf{75.77} & \textbf{92.71\%} & \\
\cmidrule(l){2-8}

& \multirow{2}{*}{$0.63$}
& \multirow{2}{*}{8.30}
  & CORP   & --    & 55.90 & 68.40\% & \multirow{2}{*}{\sgreen{\textbf{20.05}}} \\
& & & STARFISH & 5,000 & \textbf{72.29} & \textbf{88.45\%} & \\
\cmidrule(l){2-8}

& \multirow{2}{*}{$0.69$}
& \multirow{2}{*}{3.54}
  & CORP   & --    & 46.60 & 57.02\% & \multirow{2}{*}{\sgreen{\textbf{28.78}}} \\
& & & STARFISH & 5,000 & \textbf{70.12} & \textbf{85.79\%} & \\
\cmidrule(l){2-8}

& \multirow{2}{*}{$0.75$}
& \multirow{2}{*}{1.42}
  & CORP   & --    & 32.70 & 40.01\% & \multirow{2}{*}{\sgreen{\textbf{42.00}}} \\
& & & STARFISH & 5,000 & \textbf{67.03} & \textbf{82.01\%} & \\
\cmidrule(l){2-8}

& \multirow{2}{*}{$0.80$}
& \multirow{2}{*}{0.95}
  & CORP   & --    & --    & --       & \multirow{2}{*}{--} \\
& & & STARFISH & 5,000 & \textbf{64.03} & \textbf{78.34\%} & \\
\cmidrule(l){2-8}

& \multirow{2}{*}{$0.85$}
& \multirow{2}{*}{0.62}
  & CORP   & --    & --    & --       & \multirow{2}{*}{--} \\
& & & STARFISH & 5,000 & \textbf{59.70} & \textbf{73.05\%} & \\

\midrule

\multirow{12}{*}{\shortstack[l]{DeiT3-H \\(Dense Acc. 84.97)}}
& \multirow{2}{*}{$0.50$}
& \multirow{2}{*}{78.27}
  & CORP   & --    & 82.80 & 97.45\% & \multirow{2}{*}{\sgreen{\textbf{1.48}}} \\
& & & STARFISH & 9,000 & \textbf{84.06} & \textbf{98.93\%} & \\
\cmidrule(l){2-8}

& \multirow{2}{*}{$0.63$}
& \multirow{2}{*}{64.38}
  & CORP   & --    & 79.10 & 93.09\% & \multirow{2}{*}{\sgreen{\textbf{5.00}}} \\
& & & STARFISH & 9,000 & \textbf{83.35} & \textbf{98.09\%} & \\
\cmidrule(l){2-8}

& \multirow{2}{*}{$0.69$}
& \multirow{2}{*}{46.72}
  & CORP   & --    & 75.10 & 88.38\% & \multirow{2}{*}{\sgreen{\textbf{9.05}}} \\
& & & STARFISH & 9,000 & \textbf{82.79} & \textbf{97.43\%} & \\
\cmidrule(l){2-8}

& \multirow{2}{*}{$0.75$}
& \multirow{2}{*}{25.90}
  & CORP   & --    & 67.20 & 79.09\% & \multirow{2}{*}{\sgreen{\textbf{17.16}}} \\
& & & STARFISH & 9,000 & \textbf{81.78} & \textbf{96.25\%} & \\
\cmidrule(l){2-8}

& \multirow{2}{*}{$0.80$}
& \multirow{2}{*}{13.19}
  & CORP   & --    & --    & --       & \multirow{2}{*}{--} \\
& & & STARFISH & 9,000 & \textbf{80.34} & \textbf{94.55\%} & \\
\cmidrule(l){2-8}

& \multirow{2}{*}{$0.85$}
& \multirow{2}{*}{5.88}
  & CORP   & --    & --    & --       & \multirow{2}{*}{--} \\
& & & STARFISH & 9,000 & \textbf{78.24} & \textbf{92.08\%} & \\

\bottomrule
\end{tabular}
\end{table}

%% file: main.bbl
\begin{thebibliography}{37}
\providecommand{\natexlab}[1]{#1}
\providecommand{\url}[1]{\texttt{#1}}
\expandafter\ifx\csname urlstyle\endcsname\relax
  \providecommand{\doi}[1]{doi: #1}\else
  \providecommand{\doi}{doi: \begingroup \urlstyle{rm}\Url}\fi

\bibitem[Benbaki et~al.(2023)Benbaki, Chen, Meng, Hazimeh, Ponomareva, Zhao, and Mazumder]{benbaki2023fast}
Riade Benbaki, Wenyu Chen, Xiang Meng, Hussein Hazimeh, Natalia Ponomareva, Zhe Zhao, and Rahul Mazumder.
\newblock Fast as chita: Neural network pruning with combinatorial optimization.
\newblock In \emph{International Conference on Machine Learning}, pages 2031--2049. PMLR, 2023.

\bibitem[Chen et~al.(2025)Chen, Ma, Huang, Wang, Wang, Sun, Huang, and John]{chen2025optimal}
Shaowu Chen, Wei Ma, Binhua Huang, Qingyuan Wang, Guoxin Wang, Weize Sun, Lei Huang, and Deepu John.
\newblock Optimal brain connection: Towards efficient structural pruning.
\newblock \emph{arXiv preprint arXiv:2508.05521}, 2025.

\bibitem[Chen et~al.(2020)Chen, Kornblith, Norouzi, and Hinton]{chen2020simple}
Ting Chen, Simon Kornblith, Mohammad Norouzi, and Geoffrey Hinton.
\newblock A simple framework for contrastive learning of visual representations.
\newblock In \emph{International Conference on Machine Learning}, pages 1597--1607. PMLR, 2020.

\bibitem[Deng et~al.(2009)Deng, Dong, Socher, Li, Li, and Fei-Fei]{imagenet}
Jia Deng, Wei Dong, Richard Socher, Li-Jia Li, Kai Li, and Li~Fei-Fei.
\newblock Imagenet: A large-scale hierarchical image database.
\newblock In \emph{2009 IEEE Conference on Computer Vision and Pattern Recognition}, pages 248--255, 2009.

\bibitem[Dosovitskiy et~al.(2020)Dosovitskiy, Beyer, Kolesnikov, Weissenborn, Zhai, Unterthiner, Dehghani, Minderer, Heigold, Gelly, et~al.]{dosovitskiy2020image}
Alexey Dosovitskiy, Lucas Beyer, Alexander Kolesnikov, Dirk Weissenborn, Xiaohua Zhai, Thomas Unterthiner, Mostafa Dehghani, Matthias Minderer, Georg Heigold, Sylvain Gelly, et~al.
\newblock An image is worth 16x16 words: Transformers for image recognition at scale.
\newblock \emph{arXiv preprint arXiv:2010.11929}, 2020.

\bibitem[Ericsson et~al.(2022)Ericsson, Gouk, Loy, and Hospedales]{ericsson2022self}
Linus Ericsson, Henry Gouk, Chen~Change Loy, and Timothy~M Hospedales.
\newblock Self-supervised representation learning: Introduction, advances, and challenges.
\newblock \emph{IEEE Signal Processing Magazine}, 39\penalty0 (3):\penalty0 42--62, 2022.

\bibitem[Frankle and Carbin(2018)]{frankle2018lottery}
Jonathan Frankle and Michael Carbin.
\newblock The lottery ticket hypothesis: Finding sparse, trainable neural networks.
\newblock \emph{arXiv preprint arXiv:1803.03635}, 2018.

\bibitem[Frankle et~al.(2019)Frankle, Dziugaite, Roy, and Carbin]{frankle2019stabilizing}
Jonathan Frankle, Gintare~Karolina Dziugaite, Daniel~M Roy, and Michael Carbin.
\newblock Stabilizing the lottery ticket hypothesis.
\newblock \emph{arXiv preprint arXiv:1903.01611}, 2019.

\bibitem[Frantar and Alistarh(2023)]{frantar2023sparsegpt}
Elias Frantar and Dan Alistarh.
\newblock Sparsegpt: Massive language models can be accurately pruned in one-shot.
\newblock In \emph{International Conference on Machine Learning}, pages 10323--10337. PMLR, 2023.

\bibitem[Grill et~al.(2020)Grill, Strub, Altch{\'e}, Tallec, Richemond, Buchatskaya, Doersch, Avila~Pires, Guo, Gheshlaghi~Azar, et~al.]{grill2020bootstrap}
Jean-Bastien Grill, Florian Strub, Florent Altch{\'e}, Corentin Tallec, Pierre Richemond, Elena Buchatskaya, Carl Doersch, Bernardo Avila~Pires, Zhaohan Guo, Mohammad Gheshlaghi~Azar, et~al.
\newblock Bootstrap your own latent-a new approach to self-supervised learning.
\newblock \emph{Advances in Neural Information Processing Systems}, 33:\penalty0 21271--21284, 2020.

\bibitem[Han et~al.(2015)Han, Pool, Tran, and Dally]{han2015learning}
Song Han, Jeff Pool, John Tran, and William Dally.
\newblock Learning both weights and connections for efficient neural network.
\newblock \emph{Advances in Neural Information Processing Systems}, 28, 2015.

\bibitem[Hassibi and Stork(1992)]{hassibi1993second}
Babak Hassibi and David~G. Stork.
\newblock Second order derivatives for network pruning: Optimal brain surgeon.
\newblock In \emph{Advances in Neural Information Processing Systems}, volume~5, 1992.

\bibitem[Hinton et~al.(2015)Hinton, Vinyals, and Dean]{hinton2015distilling}
Geoffrey Hinton, Oriol Vinyals, and Jeff Dean.
\newblock Distilling the knowledge in a neural network.
\newblock \emph{arXiv preprint arXiv:1503.02531}, 2015.

\bibitem[Howard et~al.(2017)Howard, Zhu, Chen, Kalenichenko, Wang, Weyand, Andreetto, and Adam]{howard2017mobilenets}
Andrew~G Howard, Menglong Zhu, Bo~Chen, Dmitry Kalenichenko, Weijun Wang, Tobias Weyand, Marco Andreetto, and Hartwig Adam.
\newblock Mobilenets: Efficient convolutional neural networks for mobile vision applications.
\newblock \emph{arXiv preprint arXiv:1704.04861}, 2017.

\bibitem[Kornblith et~al.(2019)Kornblith, Norouzi, Lee, and Hinton]{kornblith2019similarity}
Simon Kornblith, Mohammad Norouzi, Honglak Lee, and Geoffrey Hinton.
\newblock Similarity of neural network representations revisited.
\newblock In \emph{International Conference on Machine Learning}, pages 3519--3529. PMLR, 2019.

\bibitem[Kusupati et~al.(2020)Kusupati, Ramanujan, Somani, Wortsman, Jain, Kakade, and Farhadi]{kusupati2020soft}
Aditya Kusupati, Vivek Ramanujan, Raghav Somani, Mitchell Wortsman, Prateek Jain, Sham Kakade, and Ali Farhadi.
\newblock Soft threshold weight reparameterization for learnable sparsity.
\newblock In \emph{International conference on machine learning}, pages 5544--5555. PMLR, 2020.

\bibitem[Kuznedelev et~al.(2023)Kuznedelev, Kurti{\'c}, Frantar, and Alistarh]{kuznedelev2023cap}
Denis Kuznedelev, Eldar Kurti{\'c}, Elias Frantar, and Dan Alistarh.
\newblock Cap: Correlation-aware pruning for highly-accurate sparse vision models.
\newblock \emph{Advances in Neural Information Processing Systems}, 36:\penalty0 28805--28831, 2023.

\bibitem[Kwon et~al.(2022)Kwon, Kim, Mahoney, Hassoun, Keutzer, and Gholami]{kwon2022fast}
Woosuk Kwon, Sehoon Kim, Michael~W Mahoney, Joseph Hassoun, Kurt Keutzer, and Amir Gholami.
\newblock A fast post-training pruning framework for transformers.
\newblock \emph{Advances in Neural Information Processing Systems}, 35:\penalty0 24101--24116, 2022.

\bibitem[LeCun et~al.(1989)LeCun, Denker, and Solla]{lecun1989optimal}
Yann LeCun, John Denker, and Sara Solla.
\newblock Optimal brain damage.
\newblock \emph{Advances in Neural Information Processing Systems}, 2:\penalty0 598--605, 1989.

\bibitem[Lucas and Mazumder(2024)]{lucas2024preserving}
Ryan Lucas and Rahul Mazumder.
\newblock Preserving deep representations in one-shot pruning: A hessian-free second-order optimization framework.
\newblock \emph{arXiv preprint arXiv:2411.18376}, 2024.

\bibitem[Malach et~al.(2020)Malach, Yehudai, Shalev-Schwartz, and Shamir]{malach2020proving}
Eran Malach, Gilad Yehudai, Shai Shalev-Schwartz, and Ohad Shamir.
\newblock Proving the lottery ticket hypothesis: Pruning is all you need.
\newblock In \emph{International Conference on Machine Learning}, pages 6682--6691. PMLR, 2020.

\bibitem[Meng et~al.(2024)Meng, Chen, Benbaki, and Mazumder]{meng2024falcon}
Xiang Meng, Wenyu Chen, Riade Benbaki, and Rahul Mazumder.
\newblock Falcon: Flop-aware combinatorial optimization for neural network pruning.
\newblock In \emph{International Conference on Artificial Intelligence and Statistics}, pages 4384--4392. PMLR, 2024.

\bibitem[Nair(2025)]{nair2025softmax}
Pravin Nair.
\newblock Softmax is $1/2$-lipschitz: A tight bound across all $\ell_p$ norms.
\newblock \emph{arXiv preprint arXiv:2510.23012}, 2025.

\bibitem[O'Shea and Nash(2015)]{cnn}
Keiron O'Shea and Ryan Nash.
\newblock An introduction to convolutional neural networks.
\newblock \emph{arXiv preprint arXiv:1511.08458}, 2015.

\bibitem[Ramanujan et~al.(2020)Ramanujan, Wortsman, Kembhavi, Farhadi, and Rastegari]{ramanujan2020s}
Vivek Ramanujan, Mitchell Wortsman, Aniruddha Kembhavi, Ali Farhadi, and Mohammad Rastegari.
\newblock What's hidden in a randomly weighted neural network?
\newblock In \emph{Proceedings of the IEEE/CVF conference on computer vision and pattern recognition}, pages 11893--11902, 2020.

\bibitem[Renda et~al.(2020)Renda, Frankle, and Carbin]{renda2020comparing}
Alex Renda, Jonathan Frankle, and Michael Carbin.
\newblock Comparing rewinding and fine-tuning in neural network pruning.
\newblock \emph{arXiv preprint arXiv:2003.02389}, 2020.

\bibitem[Romero et~al.(2015)Romero, Ballas, Kahou, Chassang, Gatta, and Bengio]{romero2015fitnetshintsdeepnets}
Adriana Romero, Nicolas Ballas, Samira~Ebrahimi Kahou, Antoine Chassang, Carlo Gatta, and Yoshua Bengio.
\newblock Fitnets: Hints for thin deep nets, 2015.
\newblock URL \url{https://arxiv.org/abs/1412.6550}.

\bibitem[Sanh et~al.(2019)Sanh, Debut, Chaumond, and Wolf]{sanh2019distilbert}
Victor Sanh, Lysandre Debut, Julien Chaumond, and Thomas Wolf.
\newblock Distilbert, a distilled version of bert: smaller, faster, cheaper and lighter.
\newblock \emph{arXiv preprint arXiv:1910.01108}, 2019.

\bibitem[Singh and Alistarh(2020)]{singh2020woodfisher}
Sidak~Pal Singh and Dan Alistarh.
\newblock Woodfisher: Efficient second-order approximation for neural network compression.
\newblock \emph{Advances in Neural Information Processing Systems}, 33:\penalty0 18098--18109, 2020.

\bibitem[Tian et~al.(2020)Tian, Krishnan, and Isola]{tian2020contrastive}
Yonglong Tian, Dilip Krishnan, and Phillip Isola.
\newblock Contrastive representation distillation.
\newblock In \emph{International Conference on Learning Representations (ICLR)}, 2020.

\bibitem[Touvron et~al.(2021)Touvron, Cord, Douze, Massa, Sablayrolles, and J{\'e}gou]{touvron2021training}
Hugo Touvron, Matthieu Cord, Matthijs Douze, Francisco Massa, Alexandre Sablayrolles, and Herv{\'e} J{\'e}gou.
\newblock Training data-efficient image transformers \& distillation through attention.
\newblock In \emph{International Conference on Machine Learning}, pages 10347--10357. PMLR, 2021.

\bibitem[Touvron et~al.(2022)Touvron, Cord, and J{\'e}gou]{touvron2022deit}
Hugo Touvron, Matthieu Cord, and Herv{\'e} J{\'e}gou.
\newblock Deit iii: Revenge of the vit.
\newblock In \emph{European Conference on Computer Vision}, pages 516--533. Springer, 2022.

\bibitem[Yu et~al.(2022)Yu, Serra, Ramalingam, and Zhe]{yu2022combinatorial}
Xin Yu, Thiago Serra, Srikumar Ramalingam, and Shandian Zhe.
\newblock The combinatorial brain surgeon: Pruning weights that cancel one another in neural networks.
\newblock In \emph{International Conference on Machine Learning}, pages 25668--25683. PMLR, 2022.

\bibitem[Zagoruyko and Komodakis(2017)]{zagoruyko2017paying}
Sergey Zagoruyko and Nikos Komodakis.
\newblock Paying more attention to attention: Improving the performance of convolutional neural networks via attention transfer.
\newblock In \emph{International Conference on Learning Representations (ICLR)}, 2017.

\bibitem[Zhang and Yang(2026)]{zhang2026corp}
Boxiang Zhang and Baijian Yang.
\newblock Corp: Closed-form one-shot representation-preserving structured pruning for vision transformers.
\newblock \emph{arXiv preprint arXiv:2602.05243}, 2026.

\bibitem[Zheng et~al.(2022)Zheng, Zhang, Yang, Tan, Xiao, Ren, Pu, et~al.]{zheng2022savit}
Chuanyang Zheng, Kai Zhang, Zhi Yang, Wenming Tan, Jun Xiao, Ye~Ren, Shiliang Pu, et~al.
\newblock Savit: Structure-aware vision transformer pruning via collaborative optimization.
\newblock \emph{Advances in Neural Information Processing Systems}, 35:\penalty0 9010--9023, 2022.

\bibitem[Zhu and Gupta(2017)]{zhu2017prune}
Michael Zhu and Suyog Gupta.
\newblock To prune, or not to prune: exploring the efficacy of pruning for model compression.
\newblock \emph{arXiv preprint arXiv:1710.01878}, 2017.

\end{thebibliography}
